\documentclass[11pt,a4paper]{article}

\usepackage{amsmath,amsfonts,amssymb,amsthm}
\usepackage{array}
\usepackage{makecell}
\usepackage{graphicx}
\usepackage{booktabs}
\usepackage{multirow}
\usepackage{subcaption}
\usepackage{hyperref}
\usepackage{float}
\usepackage{algorithm}
\usepackage{algpseudocode}
\usepackage{authblk}

\usepackage[margin=1in]{geometry}
\usepackage{setspace}
\usepackage{xcolor}
\usepackage{url}
\usepackage{cite}

\theoremstyle{plain}
\newtheorem{theorem}{Theorem}[section]

\theoremstyle{definition}
\newtheorem{definition}[theorem]{Definition}

\theoremstyle{remark}
\newtheorem{remark}[theorem]{Remark}

\begin{document}

\title{CGDT-CESN: Correlation-Guided and Dynamics-Trained Clustered Echo State Network for Multivariate Time-Series Prediction}

\author[1]{Hariharan Srinivasan\thanks{Corresponding author: \texttt{harivasan47@gmail.com}}}
\author[1]{Masoom Bhargava}
\author[1,2]{Sayan Gupta}

\affil[1]{The Uncertainty Lab, Department of Applied Mechanics \& Biomedical Engineering, Indian
Institute of Technology Madras, India -  600036 .}
\affil[2]{Center for Complex Systems \& Dynamics IIT Madras, India - 600036}

\maketitle

\begin{abstract}
Accurate prediction of multivariate time series requires models that simultaneously capture variable-specific dynamics and cross-variable interactions. In this work, a Correlation-Guided and Dynamics-Trained Clustered Echo State Network (CGDT-CESN) is proposed, which comprises a modular reservoir architecture with each cluster distinctly specializing in one of the variables in the multivariate time series. An innovative reservoir design is proposed in which inter-cluster connections are enabled by empirically observed correlations among variables, providing a data-driven coupling mechanism. Apart from conventional readout training, inter-cluster connection weights informed by correlations are trained, which in turn allows the reservoir to capture time-varying cross-variable dependencies without sacrificing computational efficiency and further improves forecasting accuracy. The prediction performance of CGDT-CESN is found to be superior to that of the conventional ESN when evaluated on two fundamentally different datasets: Twitter information cascade evolution and solar wind forecasting. Also, the improvement in performance is preserved across multiple prediction horizons, even when presented with a low volume of data. By unifying cluster-specific specialization with learned inter-variable coupling, CGDT-CESN offers a principled, data-driven route to multivariate time series prediction.

\end{abstract}


\section{Introduction}

Multivariate time series record the joint evolution of several variables and arise wherever a physical, biological, or socio-technical process is observed through multiple channels \cite{wei2019multivariate}. Examples span financial markets, industrial process monitoring, climate systems, social networks, and space-weather observations, each of which exhibits both temporal autocorrelation and cross-variable coupling \cite{tsay2013multivariate}. While univariate forecasting extrapolates a single variable from its own history, multivariate prediction must simultaneously model the dynamics of each variable and the dependence that exist between them \cite{kim2025comprehensive}.

Deep recurrent neural network (RNN) architectures such as long short-term memory (LSTM) networks and gated recurrent units (GRUs) \cite{kim2025comprehensive, kong2025deep} are widely used for this purpose. Although these architectures have proven to capture temporal patterns, their gradient-based training impose significant computational expense and requires extensive hyperparameter tuning. Reservoir computing (RC) offers an efficient alternative by replacing backpropagation through time with a fixed random recurrent layer and a single linear readout. RC avoids the vanishing- and exploding-gradient problems of RNN training while substantially reducing computational cost. The Echo State Network (ESN) is the instance of RC with a fixed random reservoir of nonlinear nodes whose fading-memory dynamics project the input into a high-dimensional feature space from which the readout is learned by ridge regression \cite{jaeger2002adaptive}. A range of reservoir architectures have since been proposed to exploit this nonlinearity more effectively.

Architectures inspired by modular organization and functional specialization from realistic biological neural networks, offers computational advantages in processing efficiency and structural robustness \cite{meunier2010modular, hilgetag2000anatomical, hilgetag2004clustered}. Building on this principle, several clustered-reservoir ESNs have been proposed: scale-free highly clustered ESNs \cite{deng2007collective}; Ward, K-medoids, and centroid-based partitions of scale-free and small-world reservoirs \cite{najibi2015scesn}, apriori data-driven multi-clustered ESN \cite{li2015priori}; and Erdős–Rényi or Barabási–Albert reservoirs applied to signal denoising and image classification \cite{junior2020clustered}. These clustered designs have, however, been developed almost exclusively for univariate prediction tasks [13], [14], with relatively little attention paid to the multivariate setting \cite{xue2016application, yu2011clustered}.

 Several ESN variants have been proposed for the multivariate setting, including the adaptive elastic ESN \cite{xu2016adaptive}, broad ESN \cite{yao2019broad}, robust variational ESN \cite{xu2018hybrid}, Laplacian ESNs \cite{han2017laplacian} and hierarchical ESNs. However, these methods significantly increase hyperparameter scaling and computational overhead. More recently, the clustered ESN (CESN) \cite{hariharan2025multivariate} has been introduced specifically for multivariate prediction: the reservoir is partitioned into a number of clusters equal to the number of target variables, with each variable mapped to a dedicated cluster. Sparse inter-cluster connections allow the clusters to exchange information and to learn cross-variable correlations. The readout is trained per cluster for segregating each cluster's specialization and to reduce interference in the output weights. Experiments on ring, Erdős–Rényi, and scale-free topologies in \cite{hariharan2025multivariate} showed that the clustered configuration predicts more accurately than a conventional ESN of equivalent size. This method has been applied to predict complex multivariate time series such as the National Stock Exchange (NSE) India, stock market, solar wind, and chaotic Rössler time series, and has been shown to enhance the prediction accuracy. A limitation of existing ESN and CESNs is that the inter-cluster connection density is a free hyperparameter that is tuned by exhaustive search rather than inferred from the data.

 The main objective of this research is to develop a data-driven hyperparameter fixation approach that is generalizable and flexible for any multivariate time series. Inferring the inter-cluster connectivity from the data becomes crucial in this approach as it indirectly establishes the strength of the dependencies between the variables in a multivariate time series. These dependencies are mathematically modeled through the correlation between the variables in the multivariate time series and are computed using the Spearman rank correlation method \cite{Spearman1904, Schober2018}. These model the inter-cluster connections and in turn mitigates the brute force search of the hyperparameter in the existing literature. The weights of these inter-cluster edges in this Correlation-Guided CESN (CG-CESN) are trained, so that the reservoir learns cross-variable couplings directly from the data while preserving the linear-readout structure that defines the ESN paradigm. To train the inter-cluster weights without recovering the gradient-based optimization cost of a RNN, an alternating coordinate-descent scheme is adopted that updates the readout and in turn the inter-cluster weights. This updation alters the dynamics of the reservoir and  the resulting model, therefore named as Correlation-Guided and Dynamics-Trained Clustered Echo State Network (CGDT-CESN), is presented in the paper. CGDT-CESN preserves the echo state property verified through \cite{jaeger2001echo} state convergence and spectral radius analysis, using correlated and uncorrelated noise probes, both before and after training. 
 
 The performance of CGDT-CESN is evaluated on two datasets presenting fundamentally different types of temporal dynamics: discrete, event-driven temporal process and continuous-valued physical time series. This complementary selection paves way to assess the generality of CGDT-CESN across both the types of multivariate time series whose underlying process is different and commonly encountered in real-world applications.  
 
 For the discrete temporal process, information cascades on Twitter \cite{adamic2016information, zhou2021survey, ma2017detect} is chosen. In the context of social network, a cascade represents the information propagation through sequence of user interactions. Cascade descriptors namely size, depth, breadth, and structural virality are estimated to capture different aspects of cascade evolution constituting to a multivariate time series. These descriptors are predicted at three growth horizons corresponding to 50\%, 75\%, and 100\% of the final cascade growth. Performance is reported as a function of training-set size, to characterize data efficiency. Solar-wind measurements are considered as representative for continuous-valued temporal process, which constitutes different physical variables of solar wind \cite{SWPC_RTSW}. Finally, the computational cost of CGDT-CESN is benchmarked against the standard ESN to quantify the overhead introduced in training.

The remainder of this paper is organized as follows: Section \ref{sec:esn} gives a brief introduction to the ESN, Section \ref{sec:cgdt-cesn} includes the necessary background with Section \ref{sec:wiring} explaining the methodology of data-driven construction of inter-cluster connections, while Section \ref{sec:cgdt-esn} explains the steps involved in the training process of CGDT-CESN algorithm. Section \ref{sec:esp} explores the theoretical properties of the proposed algorithm with particular emphasis on echo state property. The application of the algorithm on the representative examples involving information cascade in social media and solar wind datasets are presented in Section \ref{sec:apps} with discussion emphasizing about the performance. The salient outcome of this study is presented in the concluding Section \ref{sec:conc}.  

\section{Echo State Network}
\label{sec:esn}
Echo State Network (ESN) consists of three layers: input, a reservoir with $N$ dynamically nonlinear nodes, and an output layer. At each discrete time step, $n = 0,1,...,T$, the input vector $\mathbf{u}(n)\in\mathbb{R}^{N_u}$, with $N_u$ denoting the number of input variables is projected onto the reservoir through the input weight matrix $\mathbf{W}_{in}\in\mathbb{R}^{N \times N_u}$. The internal dynamics are governed by the sparse recurrent matrix $\mathbf{W}\in\mathbb{R}^{N\times N}$ and the weights are randomly initialized and scaled. The spectral radius $\rho$ measures the largest magnitude of the eigenvalues of the reservoir weight matrix which helps to determine the reservoir stability and memory dynamics. The condition, $\rho(\mathbf{W})<1$ ensures the echo state property such that the influence of past inputs decays and also prevents unstable reservoir dynamics.

The reservoir state $\mathbf{x}(n)\in\mathbb{R}^{N}$ evolves according to
\begin{equation}
\mathbf{x}(n)
=
(1-\alpha)\mathbf{x}(n-1)
+
\alpha\tanh
\left(
\mathbf{W}_{in}\mathbf{u}(n)
+
\mathbf{W}\mathbf{x}(n-1)
\right),
\label{eq:esn_state}
\end{equation}
where $\alpha\in[0,1]$ is the leak rate controlling the reservoir memory. The network output is computed as
\begin{equation}
\mathbf{y}(n)=\mathbf{W}_{out}\mathbf{x}(n),
\label{eq:esn_output}
\end{equation}
where $\mathbf{W}_{out}$ is the trainable output weight matrix. Training is performed only on $\mathbf{W}_{out}$ using ridge regression,
\begin{equation}
\mathbf{W}_{out}
=
\arg\min_{\mathbf{W}}
\left(
\|\mathbf{Y}-\mathbf{W}\mathbf{X}\|^2
+
\lambda\|\mathbf{W}\|^2
\right),
\label{eq:ridge}
\end{equation}
where the collected reservoir states are contained in  $\mathbf{X}$, the target output matrix is $\mathbf{Y}$ , and $\lambda$ is the regularization parameter. As the readout weights are optimized while the reservoir remains fixed, computational burden is sufficiently reduced as opposed to conventional recurrent neural networks. 

\section{Correlation-Guided and Dynamics Trained CESN}
\label{sec:cgdt-cesn}

\subsection{Clustered ESN}
Clustered Echo State Network (CESN) extend the conventional ESN by partioning the reservoir with $K$ clusters, each containing $N_c$ neurons. The number of clusters is fixed depending on number of variables to be predicted. Let the total number of neurons be $N = K N_c$.
Let:
\begin{itemize}
    \item $\mathbf{u}(n) \in \mathbb{R}^{N_u}$ be the input
    \item $\mathbf{x}(n) \in \mathbb{R}^{N}$ be the reservoir state
    \item $\mathbf{y}(n) \in \mathbb{R}^{K}$ be the output
\end{itemize}
The reservoir is divided into $K$ clusters, such that
\[
\mathbf{x}(n) =
\begin{bmatrix}
\mathbf{x}_1(n) \\
\mathbf{x}_2(n) \\
\vdots \\
\mathbf{x}_K(n)
\end{bmatrix}, \quad \mathbf{x}_c(n) \in \mathbb{R}^{N_c}; c= 1,2,...,K
\]
In CESN, the reservoir topology is controlled through two hyperparameters: intra-cluster connection probability $p_{in}$ among the node population within the same cluster and $p_{out}$ which determines the probability of connections between the $K$ clusters. In order to obtain a clustered topology, $p_{in} \gg p_{out}$ which results in densely connected nodes within clusters and sparse connections between the clusters. The state of the each cluster evolves according to the standard ESN dynamics (Eq. \ref{eq:esn_state}) and the readout layer is designed specific to the clusters. The cluster $c$ states and targets are stacked over time as
\begin{align}
\mathbf{X}_c &=
\begin{bmatrix}
\mathbf{x}_c(1)^T \\
\mathbf{x}_c(2)^T \\
\vdots \\
\mathbf{x}_c(T)^T
\end{bmatrix}, \quad
\mathbf{Y}_c =
\begin{bmatrix}
\mathbf{y}_c(1) \\
\mathbf{y}_c(2) \\
\vdots \\
\mathbf{y}_c(T)
\end{bmatrix}
\end{align}

For each cluster $c$, solving by ridge regression as:
\begin{align}
\mathbf{W^{out}_c} =
\mathbf{Y}_c^T \mathbf{X}_c (\mathbf{X}_c^T \mathbf{X}_c + \lambda_{out} I)^{-1}
\end{align}

leads to minimizing:
\begin{align}
\min_{\mathbf{W^{out}_c}}
\|\mathbf{Y}_c - \mathbf{X}_c \left(\mathbf{W^{out}_c}\right)^{T}\|^2 + \lambda_{out} \|\mathbf{W^{out}_c}\|^2
\end{align}
Hence, by assigning dedicated neuron groups to each input variable from a multivariate time series, CESN has demonstrated improved predictive performance \cite{hariharan2025multivariate}.  

Since the variables of the multivariate process coevolve with respect to time, they inherently exhibit mutual dependencies. In the existing CESN, the probability of inter-cluster coupling $p_{out}$ between the clusters are treated as free hyperparameter which strongly depends on user-defined choices. Furthermore, it is important to note that typically the variables present heterogeneous degree of dependence in a multivariate process, which the CESN is not equipped to capture as $p_{out}$ is treated as global hyperparameter. As a result, the resulting reservoir topology of CESN may not efficiently represent the underlying dynamics of multivariate data. 

\begin{figure*}
    \centering
        \includegraphics[width=1\linewidth]{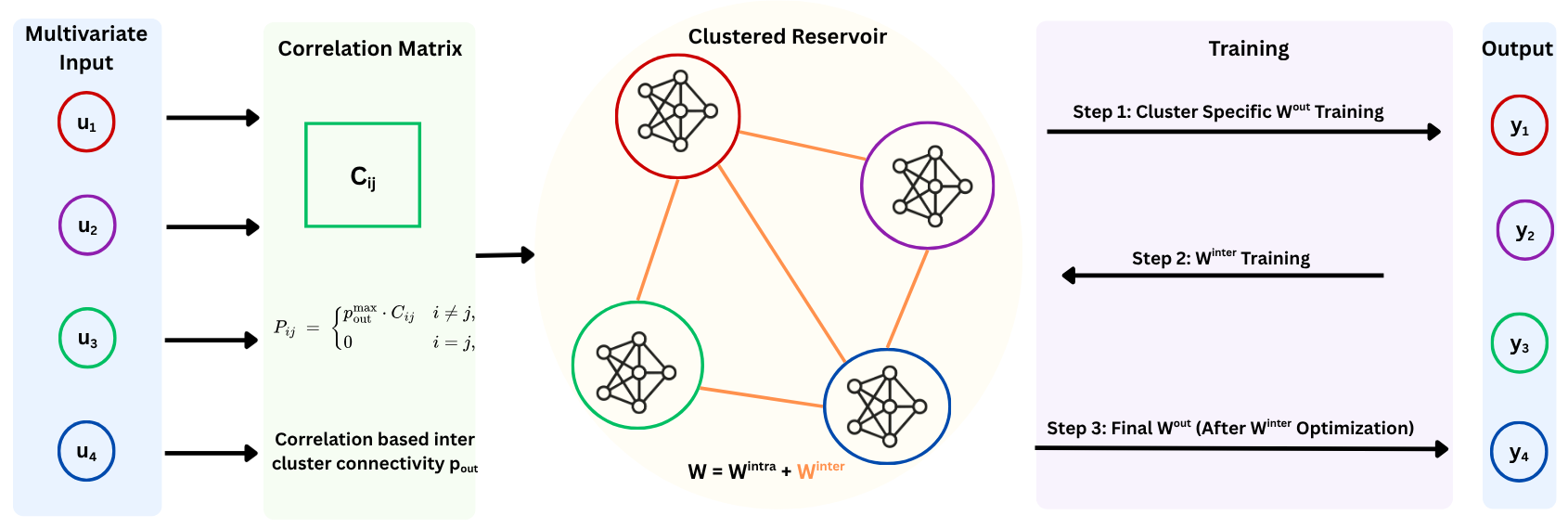}
    \caption{Architecture of Correlation-Guided and Dynamics Trained Clustered Echo state Network Illustrated for $K=4$ clusters corresponding to the four features.}
    \label{fig:cesn_architecture}
\end{figure*}
This motivates to device a strategy in which the $p_{out}$ could be inferred from data. Hence, the reservoir weight matrix will be correlation guided. Even after designing the reservoir, whether the inter-cluster connections appropriately weights the connections is still a question that needs to be examined. Standard ESN training estimates only the readout $\mathbf{W^{out}}$ by ridge regression, leaving the reservoir matrix $\mathbf{W}$ and the input weights $\mathbf{W^{in}}$ fixed. Attempts to train $\mathbf{W}$ or $\mathbf{W^{in}}$ directly reduce the model to a conventional RNN trained by backpropagation, incurring the cost and gradient pathologies that the ESN framework was designed to avoid. This study instead trains only the inter-cluster weights $\mathbf{W^{inter}}$, using a coordinate-descent scheme. This keeps the computational cost comparable to that of a standard ESN while letting the reservoir adapt its inter-cluster couplings to the data. The model is therefore named Correlation Guided and Dynamics Trained (CGDT-CESN), with $p_{out}$ derived from the correlation layer, and dynamics-trained refer to the iterative optimization of the inter-cluster weights $\mathbf{W^{inter}}$, which govern the couplings between clusters and shape the reservoir dynamics.

The CGDT-CESN architecture is illustrated in Fig. \ref{fig:cesn_architecture}.  Pairwise correlations among the input variables are estimated as described in Section \ref{sec:wiring} and used to determine the Bernoulli probabilities of the inter-cluster edges. Each input variable is then routed to its corresponding cluster, and the inter-cluster connections are instantiated according to the resulting probability matrix $P$ which will be derived as a result of correlation computation. Training then proceeds in three phases: (i) ridge regression of the cluster-specific readouts $\mathbf{W^{out}}$; (ii) iterative refinement of the inter-cluster weights $\mathbf{W^{inter}}$ via the coordinate-descent procedure detailed in Section \ref{training}; and (iii) a final readout update with the refined $\mathbf{W^{inter}}$. 

The following sections will discuss in detail on the methods and mathematical steps involved in defining this architecture.

\subsection{Data-Driven Inter-Cluster Connectivity via Rank Correlation}
\label{sec:wiring}
\label{spearman_section}

To estimate the dependence between the variables, correlation method is employed, from which the $p_{out}$ for each pair of clusters could be inferred. Before choosing any methods, it is important to understand the nature of time series in terms of stationarity such that the calculated degree of correlations are meaningful. When mean $\mu$ of a time series varies over time (non-stationary), sample correlations are dominated by shared trends rather than by genuine co-movement, making them unsuitable for inferring network structure.

Let $\mathbf{x}_n = (\mathbf{x}_{(n,1)}, \ldots, \mathbf{x}_{(n,K)})^\top \in \mathbb{R}^K$ denote a vector of
observations across $K$ variables at discrete time $n$, and let
\begin{equation}
    X = \bigl(\mathbf{x}_1,\, \mathbf{x}_2,\, \ldots,\, \mathbf{x}_T\bigr)^\top \in \mathbb{R}^{T \times K}
\end{equation}
be a multivariate time series of length $T$. Define $K$ as the number of variables (and hence clusters). A univariate series $\{\mathbf{x}_{n,k}\}$ is said
to be \emph{weakly stationary} if its first and second moments are finite and invariant
to time translation, \textit{i.e.,},
\begin{equation}
    \begin{aligned}
        \mathbb{E}[\mathbf{x}_{n,k}] &= \mu_k < \infty \\ 
     \mathrm{Cov}(\mathbf{x}_{n,k},\, \mathbf{x}_{n+h,k}) &= \gamma_k(h) \quad \forall\, n,
\end{aligned}
\end{equation}
where $\gamma_k(h)$ depends only on the lag $h$ and not on $n$. The stationarity of each time series $\{\mathbf{x}_{n,k}\}_{t=1}^T$ can be tested for using two complementary procedures: Augmented Dickey-Fuller (ADF) test and the Kwiatkowski–Phillips–Schmidt-Shin (KPSS) test. A series is treated as stationary only when both tests agree, i.e., ADF rejects its null and KPSS fails to reject its null \cite{Box2015}.

When one or more time series is found to be non-stationary, the first-difference
transformation is applied uniformly across all $K$ variables to ensure a common
analytical basis. The differenced series is defined as,
\begin{equation}
    \Delta x_{n,k} = x_{n+1,k} - x_{n,k},
    \quad n = 1,\ldots,T-1,\quad k=1,\ldots,K.
\end{equation}

For a series integrated of order one in which first order differences are taken in between the consecutive observations, denoted $\mathrm{I}(1)$, first differencing
yields a stationary $\mathrm{I}(0)$ series. This follows because if
$\mathbf{x}_{n,k} = \mathbf{x}_{n-1,k} + \epsilon_{n,k}$ with $\epsilon_{n,k}$ stationary, then
$\Delta \mathbf{x}_{n,k} = \epsilon_{n,k}$ is stationary by construction. All subsequent
analysis is conducted on $\Delta X$, with effective series length $m = T - 1$.

Standard correlation measures such as Pearson's coefficient capture only linear dependence and are sensitive to outliers and distributional assumptions, often failing to detect monotonic but nonlinear relationships. On the other hand, Spearman rank correlation \cite{Spearman1904, Schober2018} captures monotone and potentially nonlinear dependencies in a distribution-free manner.

For two differenced marginal series $\boldsymbol{\delta}_i = \Delta X_{:,i} \in
\mathbb{R}^m$ and $\boldsymbol{\delta}_j = \Delta X_{:,j} \in \mathbb{R}^m$, let
$\mathbf{r}_i = R(\boldsymbol{\delta}_i)$ and $\mathbf{r}_j = R(\boldsymbol{\delta}_j)$
denote their rank vectors, where $R(\cdot)$ maps each element to its rank among the $m$
observations, with ties (same values in the differenced series) resolved by averaging. The Spearman rank correlation between variables $i$ and $j$ is defined as the Pearson correlation of the rank vectors:
\begin{equation}
\rho_{ij} =
\frac{
\sum_{n=1}^{m}
\left(\mathbf{r}_{i,t}-\bar{\mathbf{r}}_i\right)
\left(\mathbf{r}_{j,t}-\bar{\mathbf{r}}_j\right)
}{
\sqrt{
\left(
\sum_{n=1}^{m}
\left(\mathbf{r}_{i,t}-\bar{\mathbf{r}}_i\right)^2
\right)
\left(
\sum_{n=1}^{m}
\left(\mathbf{r}_{j,t}-\bar{\mathbf{r}}_j\right)^2
\right)
}
},
\end{equation}
where $\bar{r}_k = \frac{1}{m}\sum_{n=1}^{m} \mathbf{r}_{k,t}$ is the mean rank of variable
$k$. Since the ranks $1, 2, \ldots, m$ are always a permutation of the same set
regardless of the underlying distribution, $\rho_{ij}$ is a distribution-free
statistic satisfying $\rho_{ij} \in [-1,\,1]$, with $\rho_{ij} = \pm 1$ if and only
if $\boldsymbol{\delta}_j$ is a strictly monotone function of $\boldsymbol{\delta}_i$.

The pairwise Spearman correlations are assembled into an aggregate correlation matrix
$C \in \mathbb{R}^{K \times K}$, defined entry-wise as
\begin{equation}
    C_{ij} \;=\;
    \begin{cases}
        \lvert\rho_{ij}\rvert & i \neq j, \\[4pt]
        1                    & i = j.
    \end{cases}
\end{equation}
The absolute value is used because structural coupling between clusters should reflect the strength of dependence, regardless of sign; strong positive and strong negative co-movement equally warrant strong inter-cluster connections. By construction, $C$ is
symmetric, has unit diagonal, and satisfies $C_{ij} \in [0,1]$ for all $i,j$, making
it directly interpretable as a matrix of pairwise dependence strengths.

The matrix $C$ is used to set the inter-cluster Bernoulli connection probabilities in
the reservoir. Let $p_{\mathrm{in}} \in (0,1)$ denote the intra-cluster connection
probability, which governs the density of recurrent connections within each cluster.
To preserve the modular character of the reservoir, the inter-cluster connectivity must remain structurally subordinate to the intra-cluster connectivity. Let $N$ denote total reservoir size, so that each of the $K$ clusters contains $N_c = \frac{N}{K}$ neurons.

For a node $i$ inside some cluster $c$ with $(N_c-1)$ possible neighbors and a connection probability $p_{in}$, expected intra-cluster degree can be expressed as
\begin{equation}
    \mathbb{E}[k_{in}] = (N_c-1)p_{in}.
\end{equation}
The expected inter-cluster degree with $(K-1)N_c$ nodes outside a cluster $c$ and the inter-cluster connection probability $p_{out}$ is given by,
\begin{equation}
    \mathbb{E}[k_{out}] = (K-1)N_c p_{out}
\end{equation}
The maximum permissible inter-cluster connection probability
is bounded by $\mathbb{E}[k_{out}] \leq \mathbb{E}[k_{in}]$, leading to
\begin{equation}
    (K-1)N_c p_{out} \leq (N_c-1)p_{in}
\end{equation}
For a large reservoir $(N_c-1) \approx N_c$, leading to the bound, 
\begin{equation}
    p_{\mathrm{out}}^{\max} \leq \frac{p_{\mathrm{in}}}{K - 1}.
\end{equation}
This bound ensures that the total expected number of inter-cluster connections received
by any one cluster, summed across all $K-1$ partner clusters, does not exceed its
intra-cluster connection count. The data-driven inter-cluster probability matrix is
then defined as
\begin{equation}
\label{eq: ic_wiring}
    P_{ij} \;=\;
    \begin{cases}
        p_{\mathrm{out}}^{\max} \cdot C_{ij} & i \neq j, \\[4pt]
        0                                    & i = j,
    \end{cases}
\end{equation}
so that $P_{ij} \in \bigl[0,\, p_{\mathrm{out}}^{\max}\bigr]$. Cluster pairs with
strong empirical dependence ($C_{ij} \approx 1$) are assigned near-maximal
inter-cluster probability, while approximately independent pairs ($C_{ij} \approx 0$)
remain structurally decoupled. The diagonal is zero since self-connections are handled
separately by the intra-cluster topology.

\subsection{Training the inter-cluster edge weights}
\label{sec:cgdt-esn}

After designing the reservoir using correlation guidance, one ends up with a modified model called Correlation-Guided CESN (CG-CESN). However, whether the inter-cluster communication is appropriately weighted in accordance with the data, such that it can estimate a reliable prediction which is close to the ground truth, is a question. Hence, an innovative training method is proposed that does not completely violate the conventional ESN computing method. In this framework, a coordinate descent training scheme is adopted in which it alternates between training the readout weights $\mathbf{W^{out}}$ and $\mathbf{W^{inter}}$. Conventionally, ridge regression is solved per cluster for $\mathbf{W^{out}}$. Adding to these steps, with the identified $\mathbf{W^{out}}$, the ideal reservoir state $x^*_c(n)$ will be reconstructed with respect to the optimized $\mathbf{W^{out}}$, and also, which stays close to the original reservoir state. This reconstructed state serves as a target for estimating $\mathbf{W^{inter}}$, thereby enabling an adaptation in the inter-cluster weights to the prediction objective. This iterative process of optimizing the weights in an alternative fashion is known as the coordinate descent training method.

The detailed mathematical steps involved in CGDT-CESN training framework is described as follows:

Consider a clustered Echo State Network (ESN) with $K$ clusters, each containing $N_c$ neurons. 
\subsubsection{Forward Dynamics}
The reservoir evolves as:
\begin{align}
\label{eq:x(t)}
    x(n) &= (1-\alpha)x(n-1) + \alpha \tanh(a(n)),
\end{align}
\begin{align}
\label{eq:a(t)}
a(n) &= \mathbf{W} x(n-1) + \mathbf{W_{in}} u(n), 
\end{align}
where the reservoir weight matrix is decomposed as:
\begin{equation}
\mathbf{W} =\mathbf{W^{intra}} +\mathbf{W^{inter}}.
\end{equation}
Here,
\begin{itemize}
\item $\mathbf{W^{intra}}$ contains intra-cluster connection weights,
\item $\mathbf{W^{inter}}$ contains inter-cluster connections weights.
\end{itemize}
For cluster $c$, Eq. (\ref{eq:a(t)}) becomes :
\begin{multline}    
\label{a_c}
\mathbf{a}_c(n) = \mathbf{W_{c}^{intra}} \mathbf{x}_c(n-1)
+ \sum_{i \neq c} \mathbf{W_{(c,i)}^{inter}}\mathbf{x}_i(n-1)+ \mathbf{W^{in}_c} \mathbf{u}_c(n)
\end{multline}

where $\mathbf{W^{intra}_c} \in \mathbb{R}^{N_c \times N_c}$ is the intra-cluster weight block for cluster $c$, and $\mathbf{W^{inter}_{(c,i)}}\in \mathbb{R}^{N_c \times N_c}$ is the inter-cluster weight block from cluster $i$ to cluster $c$, representing the contributions of individual components. At the cluster level, the state update for cluster $c$ reads as  
\begin{align}
\mathbf{x}_c(n)
&= (1-\alpha)\,\mathbf{x}_c(n-1) + \alpha\,\tanh(a_c(n))
\label{eq:cluster_update}
\end{align}

\subsubsection{State Collection}

The cluster $c$ states and targets (ground truth values) are stacked over time as
\begin{align}
\mathbf{X}_c &=
\begin{bmatrix}
\mathbf{x}_c(1)^T \\
\mathbf{x}_c(2)^T \\
\vdots \\
\mathbf{x}_c(T)^T
\end{bmatrix}, \quad
\mathbf{Y}_c =
\begin{bmatrix}
\mathbf{y}_c(1) \\
\mathbf{y}_c(2) \\
\vdots \\
\mathbf{y}_c(T)
\end{bmatrix}
\end{align}

\subsubsection{Training}
\label{training}
In order to improve the validation accuracy, the following steps are followed iteratively.

\textit{{Step 1: Training Output Weights}} -  For each cluster $c$, solve ridge regression:
\begin{align}
\mathbf{W^{out}_c} =
\mathbf{Y}_c^T \mathbf{X}_c (\mathbf{X}_c^T \mathbf{X}_c + \lambda_{out} I)^{-1}
\end{align}
This minimizes:
\begin{align}
\min_{\mathbf{W^{out}_c}}
\|\mathbf{Y}_c - \mathbf{X}_c \left(\mathbf{W^{out}_c}\right)^{T}\|^2 + \lambda_{out} \|\mathbf{W^{out}_c}\|^2
\end{align}
This cluster-wise training strategy helps to model each variable with good accuracy.

\textit{{Step 2: Ideal State Reconstruction}} -  After the training of output weights, the next step is to determine a reservoir state that can reproduce the desired output.  In order to obtain a unique and temporally consistent estimate of a reservoir state to match the target output, a regularized inverse problem is solved which uses the reservoir state at the previous time step ($\mathbf{x}_c(n-1)$) as a reference. This reconstructed state  $\mathbf{x}_c^*(n)$ is chosen to reproduce the target output while remaining close to the previous reservoir state. 

For each time $n$, define:
\begin{align}
\mathbf{x}_c^*(n) =
\arg\min_\mathbf{x}
\|\mathbf{W^{out}_c} \mathbf{x}- \mathbf{y}_c(n)\|^2
+
\lambda_{inv} \|\mathbf{x}- \mathbf{x}_c(n-1)\|^2
\end{align}

Solving:
\begin{align}
    \left(\left(\mathbf{W^{out}_c}\right)^{T} \mathbf{W^{out}_c} + \lambda_{inv} I\right)\mathbf{x}_c^*(n)
= \left(\mathbf{W^{out}_c}\right)^{T} \mathbf{y}_c(n) 
+ \lambda_{inv} \mathbf{x}_c(n-1)
\end{align}
which leads to,
\begin{align}
\mathbf{x}_c^*(n) =
\left(\left(\mathbf{W^{out}_c}\right)^{T} \mathbf{W^{out}_c} + \lambda_{inv} I\right)^{-1}
\left(
\left(\mathbf{W^{out}_c}\right)^{T} \mathbf{y}_c(n)
+ \lambda_{inv} \mathbf{x}_c(n-1)\right)
\end{align}

\textit{Step 3: Inversion of Reservoir Dynamics} -- The reconstructed state provides the information about desired reservoir activity, but the inter-cluster weights that are intialized appear in the reservoir pre-activation term (Eq. \ref{a_c}). Hence, the reservoir dynamics have to be inverted to obtain the desired pre-activation vector.

From:
\begin{align}
\mathbf{x}_c(n) = (1-\alpha)\mathbf{x}_c(n-1) + \alpha \tanh(\mathbf{a}_c(n))
\end{align}

one gets,
\begin{align}
\tanh(\mathbf{a}_c^*(n)) =
\frac{\mathbf{x}_c^*(n) - (1-\alpha)\mathbf{x}_c(n-1)}{\alpha}
\end{align}

\begin{align}
\label{eq: ac_star}
\mathbf{a}_c^*(n) =
\tanh^{-1}
\left(
\frac{\mathbf{x}_c^*(n) - (1-\alpha)\mathbf{x}_c(n-1)}{\alpha}
\right)
\end{align}

In using Eq. \ref{eq: ac_star}, the argument of $\tanh^{-1}$ must lie in $(-1+\epsilon, 1-\epsilon)$, where $\epsilon$ is small to ensure numerical stability since the function is defined only between (-1,1). In the numerical calculations, $\epsilon$ is taken to be 0.01.

\textit{Step 4: Residual Construction} -- The reconstructed pre-activation vector from Eq. (\ref{eq: ac_star}) consists of the contributions from the intra, inter-cluster dynamics and external inputs. As $\mathbf{W^{intra}_c}$ and $\mathbf{W^{in}_{c}} $ are already known, they can be removed to obtain the remaining component that attributes to interactions from the other clusters. Subtracting known contributions gives,
\begin{align}
\mathbf{r}_c(n) =
\mathbf{a}_c^*(n)
-
\mathbf{W^{intra}_c} \mathbf{x}_c(n-1)
-
\mathbf{W^{in}_{c}} \mathbf{u}_c(n)
\end{align}

\textit{Step 5: Inter-Cluster Regression} -- After the computation of residual term, the next step is to determine the best inter-cluster weights that explains this contribution. For a given cluster $c$, the reservoir states of all other clusters are concatenated. Defining
Define:
\begin{align}
\mathbf{X}^{\mathbf{inter}}(n) =
\begin{bmatrix}
\mathbf{x}_1(n-1) \\
\vdots \\
\mathbf{x}_{c-1}(n-1) \\
\mathbf{x}_{c+1}(n-1) \\
\vdots
\end{bmatrix},
\end{align}
it follows that,
\begin{align}
\mathbf{r}_c(n) = \mathbf{W^{inter}_{c}} \mathbf{X}^{\mathbf{inter}}(n).
\end{align}
Collecting the residuals and inter-cluster states over all training time steps yields the matrix equation,
\begin{align}
R_c = \mathbf{X}^{\mathbf{inter}} \left(\mathbf{W^{inter}_{c}}\right)^{T},
\end{align}
where $\mathbf{W^{inter}_{c}}$ is obtained by solving a ridge regression problem. A regularization term $\lambda_{inter}$ is added which constrains the magnitude of the estimated inter-cluster weights, improving numerical stability and reducing overfitting. The resulting solution is:
\begin{align}
\mathbf{W^{inter}_{c}}= (\mathbf{R_c}^T \mathbf{X}^{\mathbf{inter}})
((\mathbf{X}^{\mathbf{inter}})^T \mathbf{X}^{\mathbf{inter}} + \lambda_{inter} I)^{-1}
\end{align}
This procedure directly estimates the inter-cluster couplings required to reproduce the desired reservoir dynamics inferred from the target outputs and introduce the coordinate descent training scheme in which it alternates between training $\mathbf{W^{out}}$ and $\mathbf{W^{inter}}$. Hence, it adds one more parameter which is the number of epochs ($E$). 

\section{Echo State Property of the CGDT-CESN}
\label{sec:esp}
This section investigates the echo state property (ESP) of the proposed CGDT-CESN framework. The supporting theorems and analytical derivations for ESP are presented in  Appendix \ref{app:props_esn}.

\subsection{Input Signal Generation and Experimental Configurations}
To evaluate whether the proposed model holds properties such as convergence of reservoir states and stability as proposed in Eq. (\ref{eq:esp_def}) and (\ref{eq:necessary}) (Appendix \ref{app:props_esn}), the reservoir is introduced with a qualitatively different class of input signals $\mathbf{u}(n) \in \mathbb{R}^d$.   

\paragraph{Uncorrelated input:}
The baseline input is drawn from an independent and identically distributed (i.i.d.)\ isotropic Gaussian,
\begin{equation}
    \mathbf{u}_{iid}(n) \overset{\text{i.i.d.}}{\sim}
    \mathcal{N}\!\left(\mathbf{0},\, \mathbf{I}_d\right),
    \quad n = 1, \dots, T,
\end{equation}
so that all components are mutually independent at every time step,
$\operatorname{Cov}[\mathbf{u}(n)] = \mathbf{I}_d$.
This i.i.d. multivariate signal ($\mathbf{u}_{iid}$) serves as a worst-case universal probe of the reservoir dynamics, carrying no exploitable correlation structure.

\paragraph{Correlated input:}
The correlated input is drawn from a multivariate Gaussian with a
prescribed covariance structure,
\begin{equation}
    \mathbf{u}(n) \sim \mathcal{N}\!\left(\mathbf{0},\,
    \boldsymbol{\Sigma}\right),
    \qquad
    \boldsymbol{\Sigma} \in \mathbb{R}^{d \times d},
    \quad \boldsymbol{\Sigma} \succ 0.
\end{equation}
Correlated samples  are generated by simulating independent standard Gaussian, $\mathbf{z}(n) \sim \mathcal{N}(\mathbf{0}, \mathbf{I}_d)$ variables and applying the transformation,
\begin{equation}
    \mathbf{u}_{\mathrm{corr}}(n) \;=\; \mathbf{L}\,\mathbf{z}(n)
\end{equation}
where $\mathbf{L}$ is the lower triangular matrix obtained by Cholesky decomposition of  $\boldsymbol{\Sigma}$.

\textit{Self-consistency:}
The correlated signal $\{\mathbf{u}_{\mathrm{corr}}(n)\}_{n=1}^{T}$ serves a dual role: it is used to compute both the Spearman-difference correlation matrix that determines the inter-cluster wiring $P_{{ij}}$ (Section~\ref{spearman_section}), and as the correlated evaluation input.
Evaluation samples are drawn from the empirical covariance
\begin{equation}
    \hat{\boldsymbol{\Sigma}}
    \;=\;
    \frac{1}{T}\,\mathbf{U}_{\mathrm{corr}}^{\top}
    \mathbf{U}_{\mathrm{corr}},
\end{equation}
replacing $\mathbf{L}$ with the Cholesky factor $\hat{\mathbf{L}}$ of $\hat{\boldsymbol{\Sigma}}$. This ensures that the inter-cluster connectivity and the driving signal share the same statistical structure, forming a fully self-consistent pipeline.

\paragraph{Pre- and Post-Training Evaluation:}
The ESP is verified with the following diagnostics state convergence test and largest Lyapunov exponent (LLE) evaluated for two different signals: uncorrelated signal, correlated signal under two training stages: before and after training $\mathbf{W^{inter}}$, yielding the following configurations:
$(\mathbf{u}_{\mathrm{iid}},\;\text{before training})$,
$(\mathbf{u}_{\mathrm{corr}},\;\text{before training})$, and
$(\mathbf{u}_{\mathrm{corr}},\;\text{after training})$.

\emph{Before training} refers to the reservoir state immediately after the data-driven inter-cluster wiring step: the recurrent weight matrix $\mathbf{W}$ has been rewired via the Spearman-difference procedure, and the readout is fitted by standard ridge regression.
\emph{After training} refers to the reservoir state following the full $\mathbf{W^{inter}}$-regularized ridge procedure, which additionally penalises abrupt changes in the readout weights across successive epochs.

In order to maintain consistency, the input signal is held fixed between training stages. Hence, any observed changes in LLE between two stages is therefore attributable solely due to the change in the readout fitting procedure.

\paragraph{State convergence as an ESP Diagnostic:}
\begin{figure}
    \centering
    \includegraphics[width=0.6\linewidth]{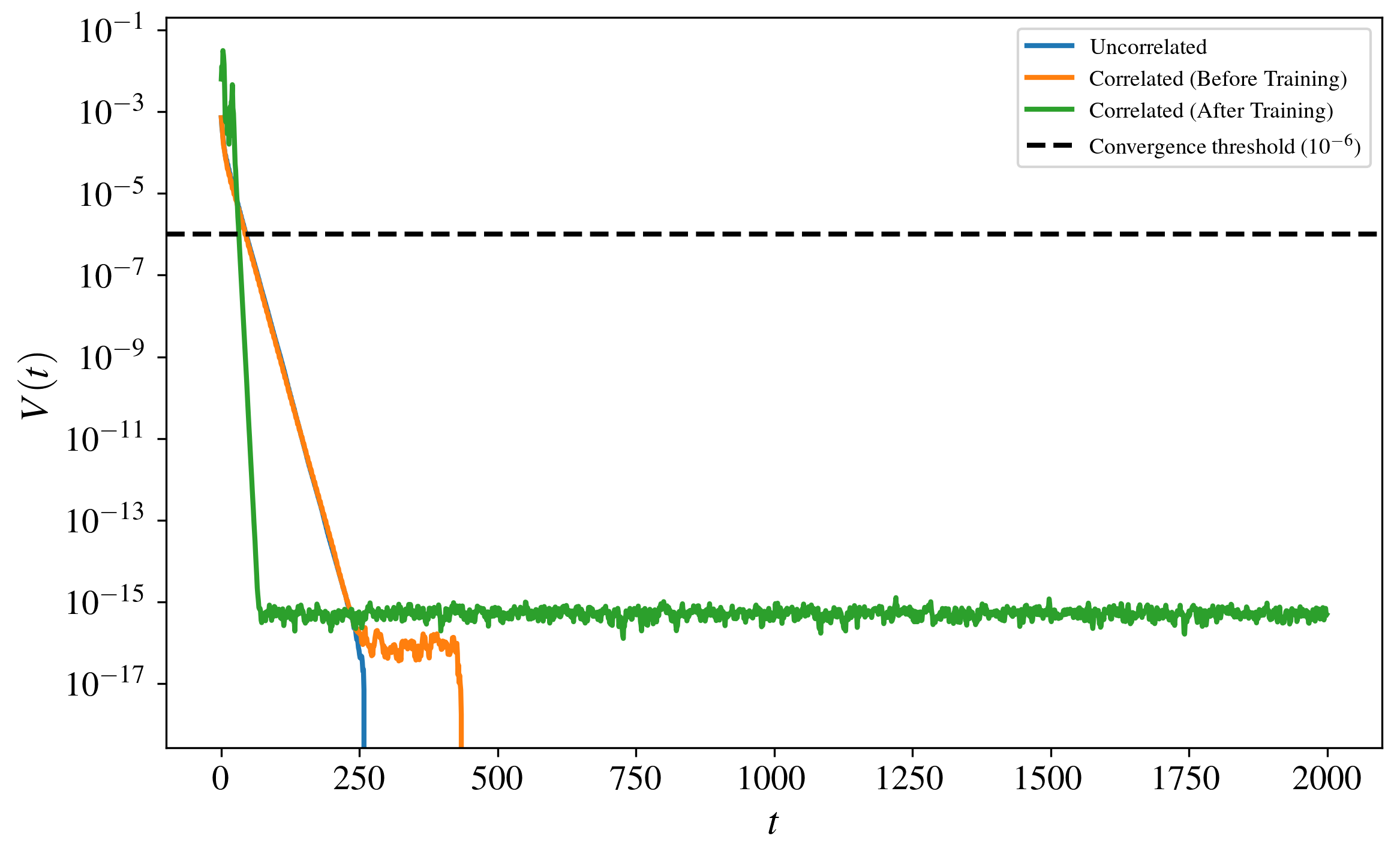}
    \caption{Convergence test of the reservoir states with respect to the Uncorrelated and correlated noise sources before and after training. }
    \label{fig:lyap_divergence}
\end{figure}

To directly assess whether the reservoir satisfies the Echo State
Property, state convergence is computed which should satisfy Eq. (\ref{eq:esp_def}). In order to study the state convergence, two reservoir trajectories $\mathbf{x}_1(n)$ and $\mathbf{x}_2(n)$ are initialized from distinct random initial conditions,

\begin{equation}
    \mathbf{x}_1(0),\; \mathbf{x}_2(0)
    \;\overset{\text{i.i.d.}}{\sim}\;
    \mathcal{U}([-0.5,\; 0.5]^{N}),
\end{equation}
driven by same input sequence. If the difference between both the states after certain time converges near zero, it can be concluded that the reservoir is dependent on present inputs and free from the influence different initial conditions and past inputs. To quantify the convergence, Euclidean distance between the two trajectories at each step is computed as given by,
\begin{equation}
    V(n) \;=\; \left\|\mathbf{x}_1(n) - \mathbf{x}_2(n)\right\|_2,
    \quad n = 1, \dots, T.
\end{equation}

 In practice, convergence is declared when $V(n)$ falls below a threshold $\varepsilon = 10^{-6}$, and the convergence steps $n^{*} = \min\{n : V(n) < \varepsilon\}$ is recorded as a complementary scalar summary. The reservoir states that do not contract or remain bounded away from zero violate the ESP, and hence its predictions become dependent on the arbitrary choice of the initial state rather than on the input history alone.
 
 In Fig. \ref{fig:lyap_divergence}, the dashed black line represents the convergence threshold and the corresponding solid lines, blue represents uncorrelated signal, and orange and green solid lines represent correlated signals before and after training. In all three cases, the reservoir states falls below the threshold after some transient steps $n$, clearly indicating that the ESP holds if the reservoir is introduced with uncorrelated/correlated signals before training. Even the trained reservoir passed state convergence test by converging near zero, indicating that training $\mathbf{W^{inter}}$ does not affect the ESP.
\begin{figure}
    \centering
    \includegraphics[width=0.6\linewidth]{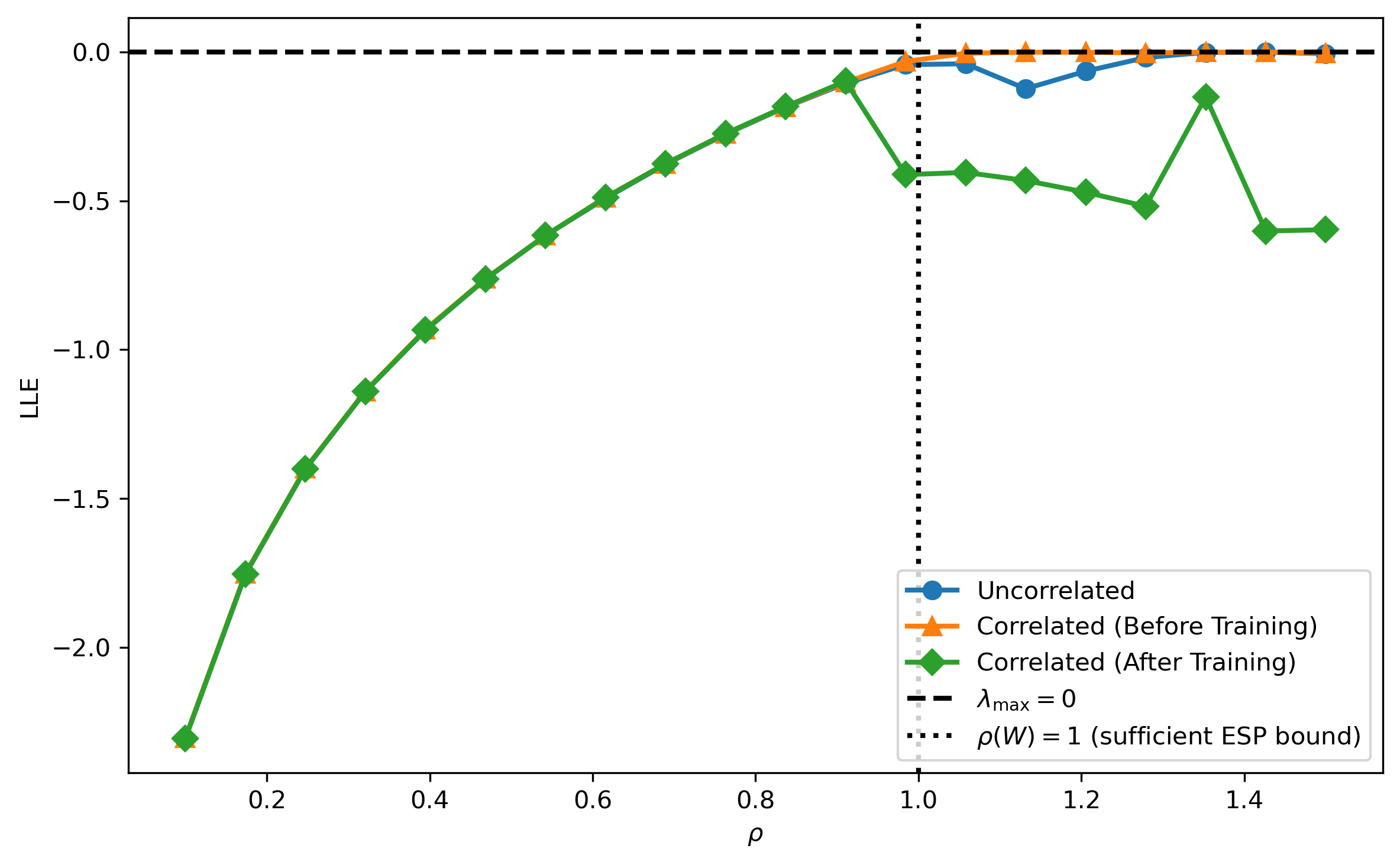}
    \caption{Largest Lyapunov Exponent (LLE) computed for varying spectral radius ($\rho$) to verify the ESP for uncorrelated and correlated input signals before and after training the reservoir. }
    \label{fig:spectral_esp}
\end{figure}

Next, a spectral radius sweep is performed to investigate the stability and ESP of the reservoir through the LLE. The variation of LLE with varying spectral radius characterizes the maintenance of ESP and also identifies the critical point where ESP is lost or divergent reservoir dynamics is observed. Fig. \ref{fig:spectral_esp} represents the computation of the LLE with respect to the varying spectral radius swept over $\rho \in [0.1,\, 1.5]$ (20 uniformly spaced values) with other hyperparameters held fixed for the four configurations represented as, uncorrelated signal ($\bullet$), correlated before($\blacktriangle$) and after ($\blacklozenge$) training. The dashed horizontal lines represents the boundary $\lambda_{max} = 0$ above which the reservoir becomes unstable. The vertical horizontal dotted line is a reference line drawn at $\rho =1$, below which the reservoir could have ESP. As shown in Fig. \ref{fig:spectral_esp}, the ESP holds for all configurations when $\rho < 1$, consistent with the theoretical bound established in Theorem. (\ref{thm:sufficient}). At $\rho \ge 1$, the ESP is violated in the untrained reservoir under both correlated and uncorrelated inputs before training, as demonstrated by the largest Lyapunov exponent $(\lambda_1)$ becoming non-negative, indicating dynamical instability.  In contrast, $\mathbf{W^{inter}}$ trained reservoir with introduced correlated-signal maintains the stability and even holds ESP even after $\rho \ge 1$. This suggests that coordinate-descent training of $\mathbf{W^{inter}}$ regularizes the reservoir dynamics without destabilizing the network.

\section{Applications}
\label{sec:apps}
With the developed data-driven and improved training methodology, the proposed algorithm is applied to real-world problem whose process can be understood with multiple variables realized over time. Particularly, the algorithm is tested on two different multivariate time series: discrete event-driven, and continuous-valued temporal process because such recordings are prevalent in nature. One such challenging and complex prediction task is information cascades in social network which is discrete and event driven in nature. 

\subsection{Information Cascades from Social Media}
In social networks, a cascade emerges as an information propagates through series of  user engagements. Prediction of information cascade remains challenging due to the complex and nonlinear nature of propagation in the social network. In order to study the nature of such complex processes, cascade trees were collected from Twitter benchmark datasets \textbf{Twitter 15/16} \cite{ma2017detect}. Each cascade in the dataset represents a single claim (starting from a source node/user) that propagated through the network via retweets (consequent directed connections). The details of the dataset has been summarized in Table \ref{tab:twitter_data}.
\begin{table}[h]
\caption{Summary of datasets and cascade labels \cite{ma2017detect}}
\centering
\begin{tabular}{lc}
\toprule
\textbf{Dataset} & \textbf{Cascades}  \\
\midrule
Twitter15 & 1,490 \\
Twitter16 & 818\\
\midrule
Combined  & 2,308  \\
\bottomrule
\end{tabular}
\label{tab:twitter_data}
\end{table}

From a cascade, features such as size, breadth, depth and structural virality can be extracted associated with timestamps, thereby enabling the construction of multivariate temporal data. The mathematical forms and its associated meaning can be found in Table \ref{tab:cascade_measures}. Here, $G(n)=(Z(n),E(n))$ denotes the cascade graph at time $n$, $Z(n)$ is the set of activated nodes, $E(n)$ is the set of diffusion edges, $d(z_i,z_j)$ is the shortest-path distance between nodes $z_i$ and $z_j$, and $L_k(n)$ denotes the set of nodes at propagation level $k$. 
\begin{figure}[h!]
    \centering
    \includegraphics[width=0.6\linewidth]{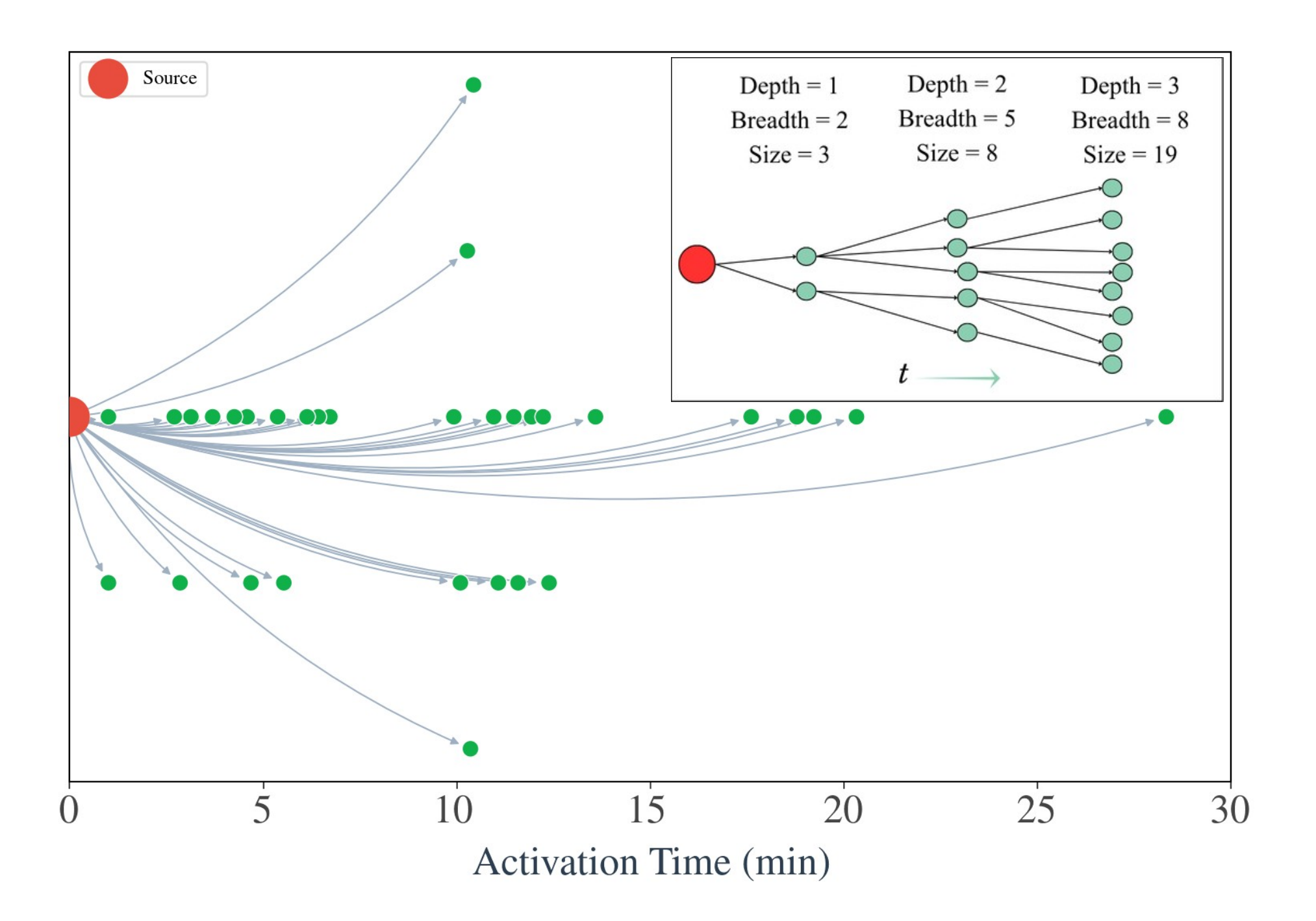}
    \caption{Time evolution of cascade with multiple evolving variables reflecting different aspects of a cascade such as size, breadth and depth.  }
    \label{fig:cascade_diag}
\end{figure}

Fig. \ref{fig:cascade_diag} depicts a representative cascade from the collected dataset, where a source user (red circle) initiates a post and subsequently other users (depicted in green) engage with the source over time (activation time). The inset of Fig. \ref{fig:cascade_diag} represents a schematic illustration of a cascade in which the associated features such as size, breadth and depth are incrementing with respect to time. It is important to note that these features evolve following discrete user activation events, and therefore, cascading process can be regarded as discrete in nature.

The task of predicting the multivariate features of a cascade with respect to time is difficult due to the complexity and variability that each individual information could possess. Hence, prediction of such a complex multivariate time series would be helpful in analyzing the trend and improve the knowledge basis on how a particular information could impact in a social network beforehand. 

\begin{table*}[ht]
\centering
\caption{Time-dependent cascade descriptors and their mathematical definitions.}
\label{tab:cascade_measures}
\renewcommand{\arraystretch}{1.4}
\begin{tabular}{|c|c|p{6cm}|}
\hline
\textbf{Measure} & \textbf{Mathematical Form} & \textbf{Meaning} \\
\hline

Cascade Size &
$S(n)=|Z(n)|$ &
Total number of activated nodes in the cascade up to time $n$. Measures the overall spread. \\
\hline

Cascade Depth &
$D(n)=\max\limits_{z\in Z(n)} d(z_0,z)$ &
Maximum shortest-path distance from the source node $z_0$ to any node in the cascade. Measures the propagation depth. \\
\hline

Cascade Breadth &
$B(n)=\max\limits_k |L_k(n)|$ &
Maximum number of nodes present at any propagation level $k$. Measures the widest parallel spread. \\
\hline

Structural Virality &
$Z_s(n)=\dfrac{1}{|Z(n)|(|Z(n)|-1)}
\sum\limits_{i\neq j} d(z_i,z_j)$ &
Average pairwise shortest-path distance among all nodes. Measures how structurally viral the cascade is. \\
\hline

\end{tabular}
\end{table*}

The dataset consists of 2,308 cascades combining Twitter 15 and Twitter 16, from which a multivariate time series consisting of variables size, breadth, depth and virality can be extracted for each cascade. The data is partitioned into training and testing sets in the ratio of 80\% and 20\%, respectively. In each cascade, only 30\% of the total events are taken into account, since the temporal duration varies across cascades. To analyze the efficiency of the model with respect to the training data volume, the effective training set size is systematically varied from 10\% to 100\% of the training pool, corresponding to 8\% to 80\% of the total dataset. The model is evaluated on a fixed test set for each configuration. Furthermore, the task for the model is designed to predict at three different horizons corresponding to 50\%, 75\% and 100\% of the final cascade evolution. This assessment enables the understanding of data efficiency as well as the impact of prediction horizon on the accuracy of cascade evolution modeling.

To understand the advantages over the conventional ESN, a sensitivity analysis is performed for ESN, \textbf{Correlation Guided CESN (CG-CESN)}(Correlation guidance only) and \textbf{Correlation Guided and Dynamics Trained-CESN (CGDT-CESN)} which operates on both correlation guidance to establish inter-cluster connections as well as trains the weights $\mathbf{W^{inter}}$. The comparison is made to understand the  effects of the additional computational layers added in the CESN framework. The performance of the model is assessed with the metric normalized root mean square (NRMSE), expressed as,

\begin{equation}
\mathrm{NRMSE} =
\sqrt{
\frac{1}{N\,\sigma_\mathbf{y}^2}
\sum_{i=1}^{N} \left(\mathbf{y}_i -\mathbf{\hat{y}}_i\right)^2
}
\end{equation}

\noindent
where $\mathbf{y}_i$ denotes the actual target value, $\hat{y}_i$ represents the predicted value, $N$ is the total number of data samples, and $\sigma_\mathbf{y}^2$ is the variance of the target signal. The NRMSE quantifies the prediction error normalized by the variability of the target data, with lower values indicating better prediction accuracy. Since the time series is multivariate and to obtain a notion of model's performance within a single metric, the average of normalized root mean square ($\xi$ =$\overline{NRMSE}$) is computed.

\begin{table*}[h]
\centering
\caption{Hyperparameters used for different reservoir computing models}
\begin{tabular}{lcccc}
\hline
\textbf{Hyperparameter} & \textbf{ESN} & \textbf{CG-CESN} & \textbf{CGDT-CESN} \\
\hline

Reservoir size ($N_{\text{res}}$)        & 200    & --   & --   \\
Neurons per cluster ($N_c$)              & --      & 50   & 50   \\
Number of clusters ($K$)                 & --      & 4    & 4    \\
Spectral radius ($\rho$)                 & 0.9   & 0.9  & 0.9  \\
Leak rate ($\alpha$)                     & 0.3   & 0.3  & 0.3  \\
Reservoir connectivity ($p_{\text{rc}}$) & 0.3    & --  & -- \\
Intra-cluster connectivity ($p_{\text{in}}$)     & --   & 0.3  & 0.3  \\
Inter-cluster connectivity ($p_{\text{out}}$) & --  & data-driven & data-driven \\
Max inter-cluster connectivity ($p_{\text{out}}^{\max}$) & -- & 0.1 & 0.1 \\
$\mathbf{W^{out}}$ regularization ($\lambda_{out}$)         & $10^{-8}$ & $10^{-8}$ & $10^{-8}$ \\
$\mathbf{W^{inter}}$ regularization ($\lambda_{inter}$)      & --  & -- & $10^{-3}$ \\
$\mathbf{x}^*_c$ regularization ($\lambda_{inv}$)      & --  & -- & $10^{-3}$ \\
Number of epochs ($E$)                        & --  & -- & 10 \\

\hline
\end{tabular}
\label{tab:hyperparams}
\end{table*}

The hyperparameters are fixed as given in the Table \ref{tab:hyperparams}. In order to match the reservoir's dynamic range, all the variables are normalized between 0 to 1. All variables are normalized to the range [0, 1] to match the reservoir's dynamic range. The scaled data is the input signal which is sent to correlation layer and probability of inter-cluster connections ($p_{out}$) is decided. Since there are four variables, four clusters are present in the reservoir and the maximum $p^{max}_{out} = \frac{p_{in}}{K-1}$. With $p_{in} = 0.3$,  maximum cap for the inter-cluster connections is $p^{max}_{out} = 0.1$ such that the modular structure is retained.

\begin{figure}[h!]
    \centering
    \includegraphics[width=0.6\linewidth]{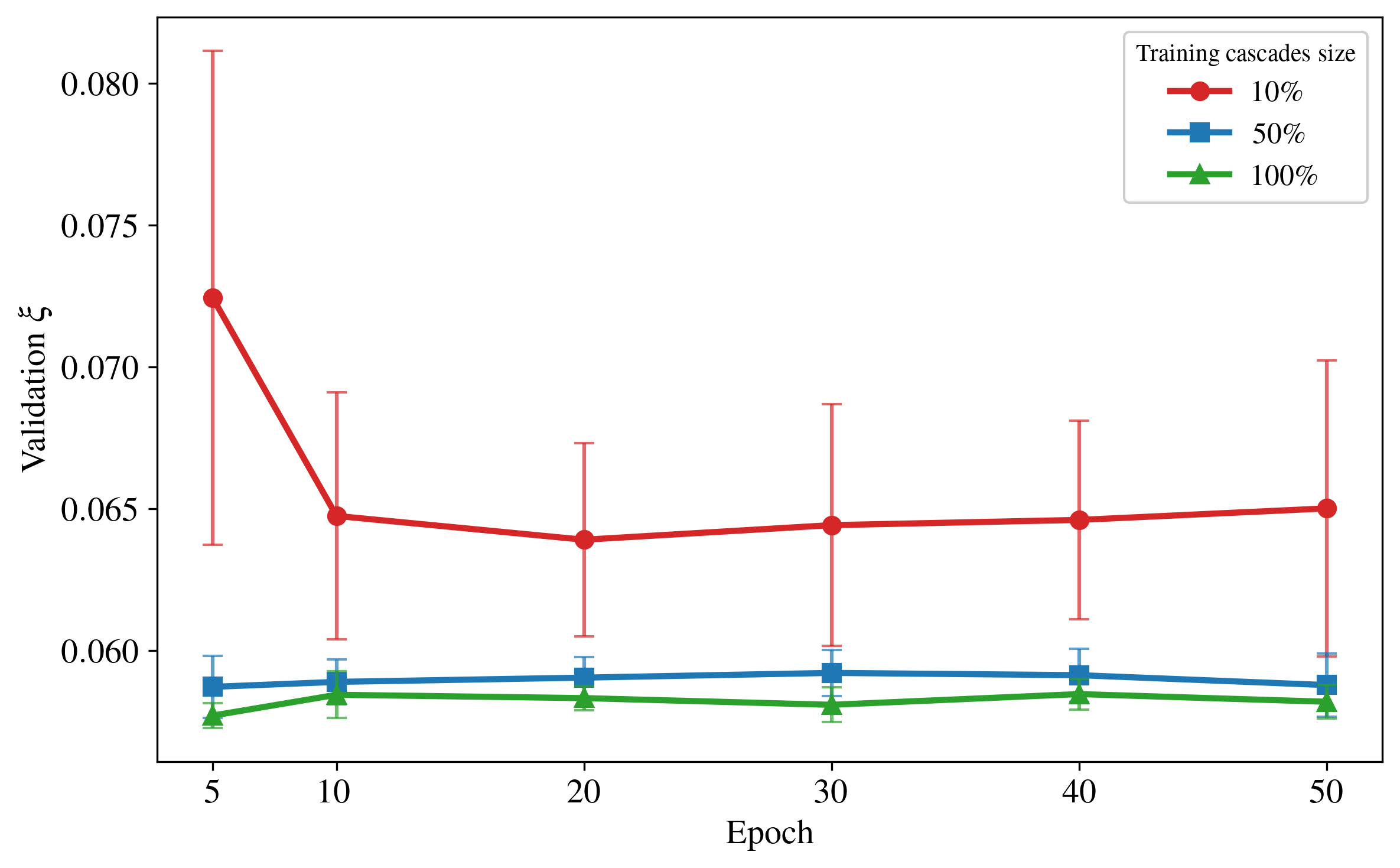}
    \caption{Validation convergence of the CGDT-\textcolor{red}{C}ESN model for different training data fractions ($10\%$, $50\%$, and $100\%$) over 50 Winter refinement epochs. The curves represent the mean validation NRMSE averaged over 10 independent runs, while the error bars denotes standard deviation. A rapid decrease in validation error is observed during the initial epochs, followed by convergence and saturation beyond approximately 10 epochs, indicating that additional refinement iterations provide only marginal improvement in predictive performance.}
    \label{fig:epoch_conv}
\end{figure}

Since the training procedure in the CGDT-CESN, the optimal number of epochs should be determined to ensure the avoidance of underfitting and overfitting of the target. The validation error ($\xi$) is computed for three different training sizes: 10\% ($\bullet$), 50\% ($\blacksquare$), and 100\% ($\blacktriangle$), as presented in Fig. \ref{fig:epoch_conv}, in order to analyze the dependence of the model performance on both the dataset size and the number of epochs. This experiment aims to identify the epoch where the validation $\xi$ converges as a result of training. It can be observed from the figure that for 10\% training size, the model performance is poor at the end of 5th epoch. However, when the number of epochs is increased to 10, $\xi$ reduced significantly and continues to remain the same for subsequent epochs. In contrast, for the 50\% and 100\% training sizes, validation $\xi$ is observed to be considerably lower than that of 10\% training size. Moreover,  convergence is achieved consistently from epoch 10 onward with smaller standard deviations. These observations indicate that the iterative $\mathbf{W^{inter}}$ training (Co-ordinate descent training) plays a more significant role when the model is trained with less samples. Hence, from this experiment, taking the computational cost into account, epoch 10 is considered to be a suitable convergence point across all training sizes. 

\begin{figure*}[h!]
        \centering
    \includegraphics[width=1.0\linewidth]{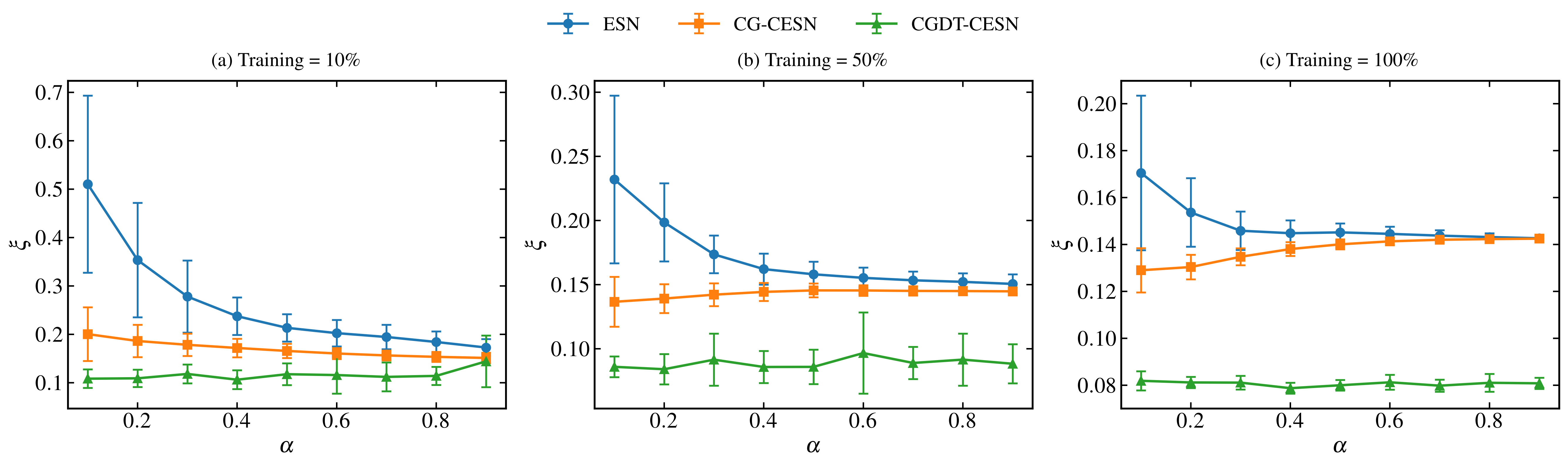}
    \caption{Effect of the leak rate parameter ($\alpha$) on the performance of ESN ($\bullet$), CG-CESN ($\blacksquare$), and CGDT-CESN ($\blacktriangle$) for training sample percentages of 10\% (a), 50\% (b), and 100\% (c). The leak rate is varied over the range ([0.1,1.0]) while all other parameters, including the number of training epochs, are kept fixed.}
    \label{fig:leak_rate}
\end{figure*}

Next, the crucial leak rate parameter ($\alpha$) is tuned over the range $[0.1, 1.0]$ for training sample percentages of 10\% (a), 50\% (b), and 100\% (c) to examine its role in ESN ($\bullet$), CG-CESN ($\blacksquare$), and CGDT-CESN ($\blacktriangle$), as shown in Fig. \ref{fig:leak_rate}. From the figures, it is evident that across all the training sample percentages CGDT-CESN demonstrates consistent and superior performance over the entire range of $\alpha$ as compared to other algorithms. CG-CESN outperforms ESN at lower $\alpha$ but converges to ESN performance as $\alpha$ increases. This analysis highlights the dependence of the algorithm on $\alpha$ that governs the past state influence on the reservoir dynamics and becomes crucial for informed parameter selection.

\begin{figure*}[h!]
        \centering
   \includegraphics[width=1.0\linewidth]{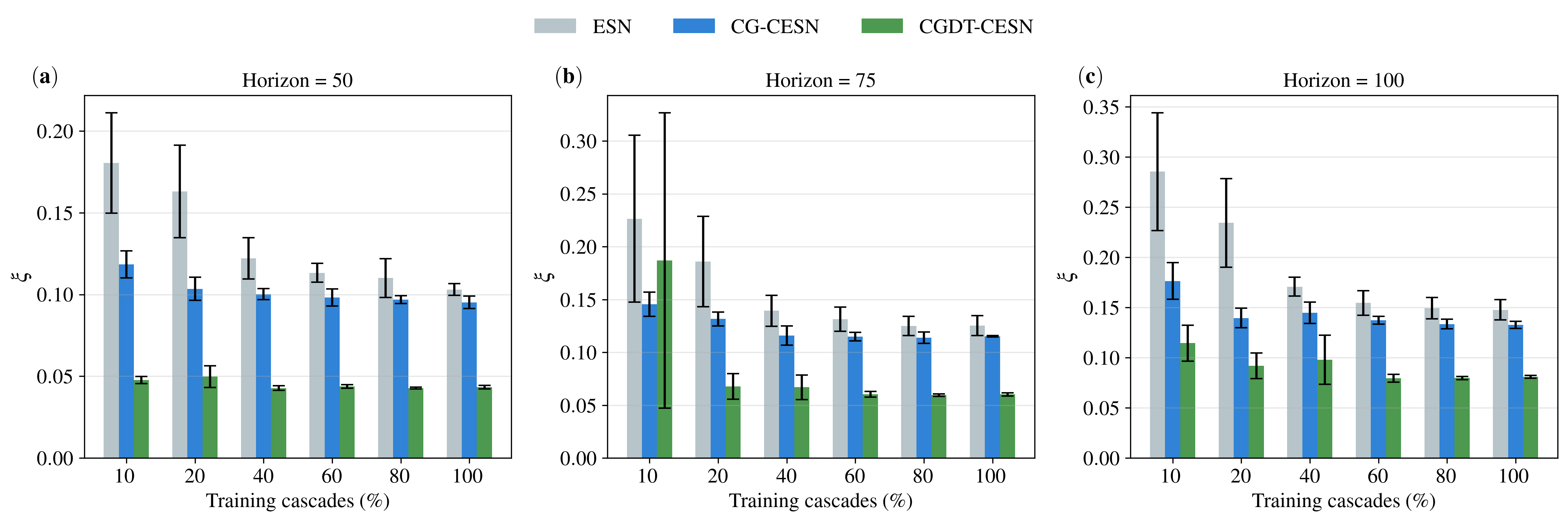}
    \caption{Performance (mean NRMSE, $\xi$) of models as a function of training data proportion (where 100 \% denotes the maximum training sample size corresponding to 80\% of the original data set) evaluated at three prediction horizons (a) 50\%, (b) 75\% and (c) 100\% for 50 realizations. The gray, blue and green bars correspond the ESN, CG-CESN and CGDT-CESN with the error bar depicting the standard deviation }    \label{fig:cascade_assesment}
\end{figure*}

Fig. \ref{fig:cascade_assesment} represents the performance of models ESN, CG-CESN, and CGDT-CESN depicted in gray, blue, and green bars with the error bar depicting the standard deviation across realizations. The assessment is based on the training data proportion ranging from 10\% to 100\%, with 100\% being the maximum training sample corresponding to 80\% of the original data set with evaluation at three prediction horizons (a) 50\%, (b) 75\%, and (c) 100\%. In all three horizons, even with a smaller number of training cascades (10\% - 20\%), the CG-CESN outperforms ESN. This highlights the advantage of data-driven correlation communication establishment between clusters, which understands the inter-relationship between the variables and enhances the prediction quality even at low percentage of training samples. Further, CGDT-CESN significantly improves the prediction accuracy due to the added training of  $\mathbf{W^{inter}}$ for 10 epochs. As the training samples increases, ESN and CG-CESN performs mostly the same which indicates that the sufficient training samples allow the conventional ESN algorithm to perform well. But the accuracy margin provided by CGDT-CESN is consistent over different training samples at all three horizons indicating the adaptation of inter-cluster weights favors and learns accordingly. The error $\xi$ increases as the prediction horizon is getting increased. The error grows mildly with prediction horizon at the 50\% horizon, the target lies on the steep growth portion of the cascade, whereas the 100\% horizon corresponds to the saturation value, which is more sensitive to small mis-estimations.  However, the increase from 50\% to 100\% is relatively small.

\begin{figure*}[h!]
        \centering
    \includegraphics[width=0.8\linewidth]{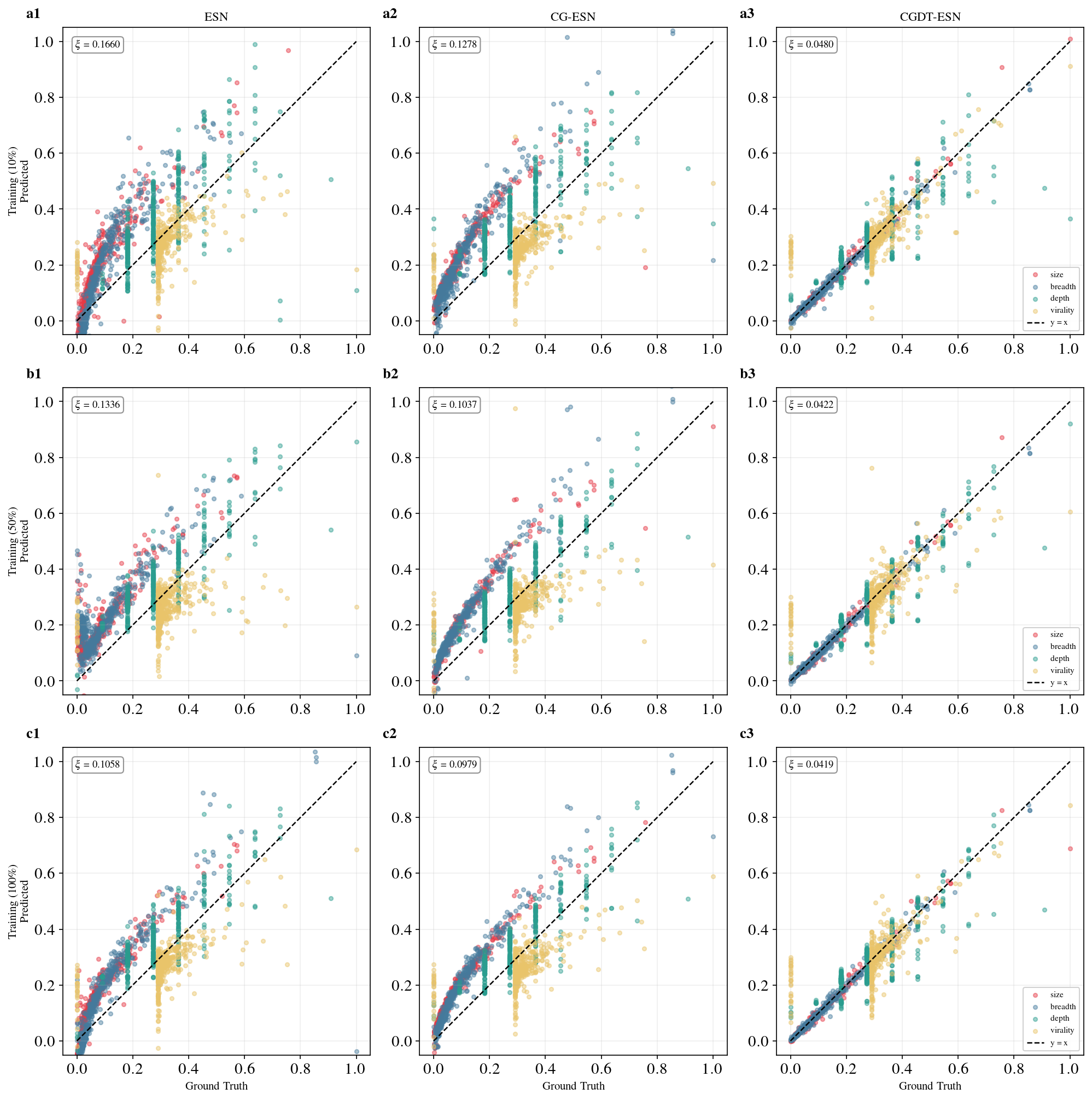}
    \caption{Scatter plots of predicted versus ground truth values for the four cascade features: size (red), breadth (green), depth (blue), and structural virality (yellow), evaluated at prediction horizon $50\%$, across three models (columns) and three training set sizes (rows). All variables are independently normalized to $[0,1]$. The dashed diagonal line ($y=x$) represents perfect prediction; points lying closer to this line indicate higher predictive accuracy. Row (a): training size $10\%$ (panels a1–a3); Row (b): training size $50\%$ (panels b1–b3); Row (c): training size $100\%$ (panels c1–c3). Column 1 (panels a1, b1, c1): ESN; Column 2 (panels a2, b2, c2): CG-CESN; Column 3 (panels a3, b3, c3): CGDT-CESN. The mean normalized root mean square error ($\xi$, averaged across all four cascade features over 50 independent realizations) is reported in the upper left corner of each subplot. Progressively tighter clustering of points around the diagonal with increasing training size reflects improved model generalization, with CGDT-CESN consistently achieving the lowest $\xi$ values, particularly under limited training data.
    }
    \label{fig:prediction}
\end{figure*}
In order to visualize the prediction performance of the models, a comparative analysis of prediction performance for different reservoir computing models, namely ESN, CG-CESN, and CGDT-CESN, across varying training sizes and prediction for horizons 50\% is presented in Fig. \ref{fig:prediction}. Each subplot illustrates the relationship between the normalized ground truth values (x-axis) and the corresponding model predictions (y-axis) for multiple cascade features, including size (red), breadth (green), depth (blue), and virality (yellow). All variables are independently normalized to the range $[0,1]$ to ensure uniform scaling across features. The dashed diagonal line, defined by $y = x$, represents the ideal case of perfect prediction, where the predicted values exactly match the ground truth. Therefore, points lying closer to this diagonal indicate higher predictive accuracy, while deviations from the line reflect prediction errors. As the number of training cascades increases, the scatter points progressively concentrate around the diagonal, demonstrating improved model generalization. Furthermore, the average normalized root mean square error, denoted by $\xi$, is reported within each subplot as a quantitative measure of prediction accuracy. Among the models, CGDT-CESN exhibits the closest alignment of points to the $y = x$ line and the lowest $\xi$ values when the model is trained with low training samples, highlighting the advantage of incorporating both data-driven connectivity and inter-cluster learning in improving predictive performance. Hence, the model achieves high accuracy with limited training data across all horizons, indicating strong sample efficiency and robust generalization.

\subsection{Solar wind dataset}
Solar wind dynamics represents as an evaluation candidate for the model in terms of continuous process, whose recordings involve multiple physical parameters being sampled with respect to time, constituting to a multivariate time series. Dataset is obtained from the Real-Time Solar wind (RSTW) database provided by the Space weather prediction Center (SWPC) \cite{SWPC_RTSW} spanning the period from 1998 to 2025. The measurements from solar wind includes density, speed, and temperature with a five day resolution. The available dataset is already in the curated form with 50th percentile (median), 1st percentile (minimum), and 99th percentile (maximum). For our experiment, maximum values of solar wind density, speed and temperature were chosen. Data pre-processing steps include normalization and smoothening of time series using a Gaussian filter.  The hyperparameters are fixed the same as in Table \ref{tab:hyperparams} except $\alpha = 0.9$, $\lambda_{inter} = 0.5$, $N_c = 100$, Number of epochs = 5 which are found to be best values for prediction.

\begin{table}[ht]
\centering
\caption{Prediction performance ($\xi$) of the models on the solar wind dataset. Results are averaged over 50 realizations and reported as mean $\pm$ standard deviation of the $\xi$.}
\label{tab:solar_nrmse}
\begin{tabular}{lc}
\hline
\textbf{Model} & \textbf{$\xi$} \\
\hline
ESN       & $0.081 \pm 0.0075$ \\
CG-CESN    & $0.067 \pm 0.0025$ \\
CGDT-CESN  & $\mathbf{0.063 \pm 0.0019}$ \\
\hline
\end{tabular}
\end{table}

The prediction performance of the models on the solar wind dataset is tabulated in Table \ref{tab:solar_nrmse} in which the averaged NRMSE ($\xi$) over the variables are computed for 50 independent realizations for each models. The results from Table  \ref{tab:solar_nrmse}  indicates that data-driven inter-cluster wiring favors the CG-CESN model to learn and capture dependencies among the variables than the conventional ESN. Further, additional training of inter-cluster weights ($\mathbf{W^{inter}}$) in CGDT-CESN improvises the performance over CG-CESN. 
\begin{figure*}[h!]
    \centering
    \includegraphics[width=1.0\linewidth]{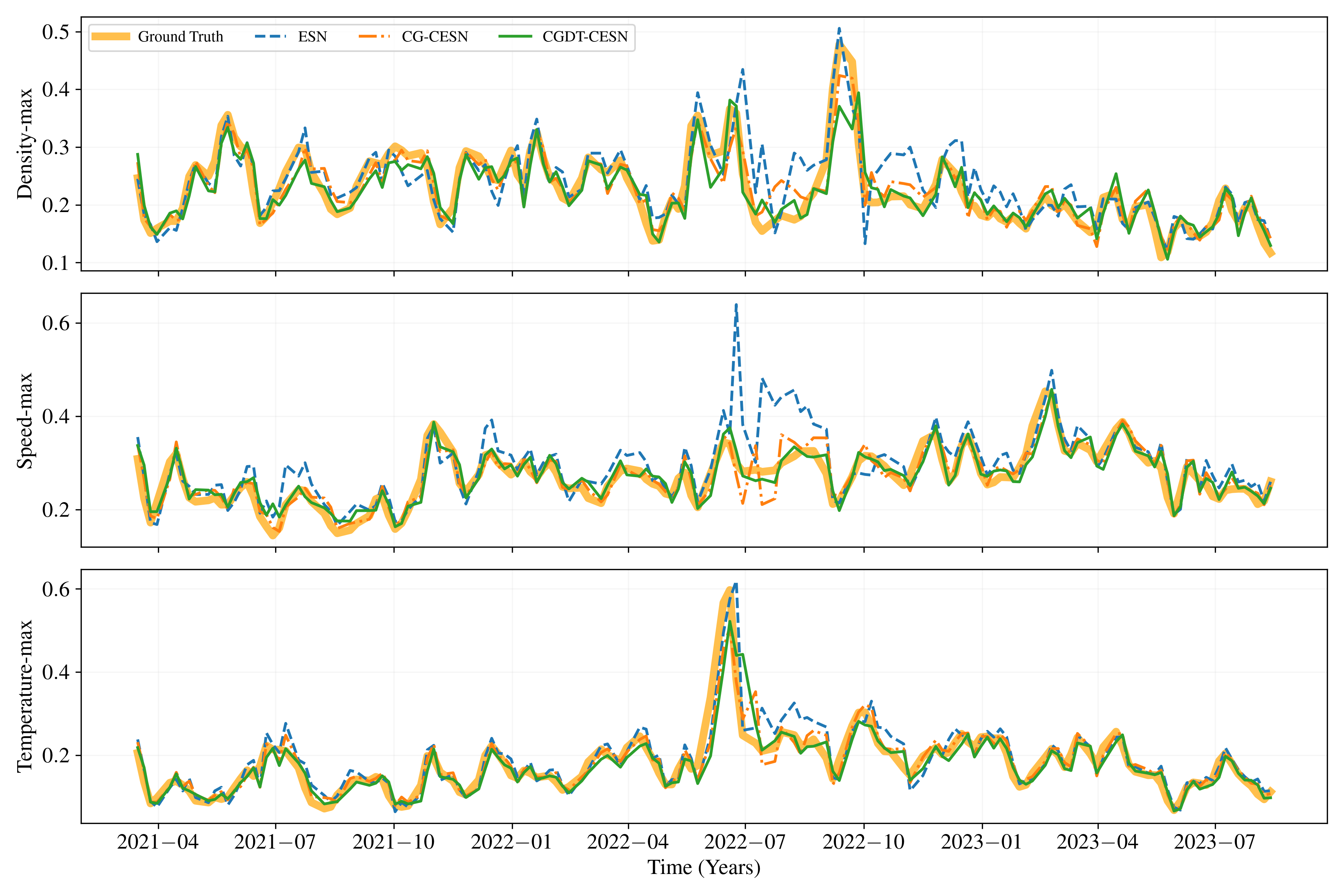}
    \caption{Visualization of solar wind prediction performance on the test set consisting of variables temperature, speed and density. The panels show the ground-truth depicted as thick orange curve and predicted values of  ESN (dashed blue line), CG-CESN (dashed dotted red line) and CGDT-CESN (solid green line). Prediction of CGDT-CESN agrees closely with the ground truth across all variables and also exhibits its stability through time as compared to ESN.}
    \label{fig:solar_pred_ts}
\end{figure*}
Fig. \ref{fig:solar_pred_ts} presents the prediction performance of the models on all the variables of the solar wind dataset with the ground truth (test set) depicted as thick orange curve and the model’s prediction are depicted as follows: ESN - dashed blue line, CG-CESN - dashed dotted red line and CGDT-CESN - solid green line. ESN tends to capture the overall trend of the time series but fails to accurately reproduce several local fluctuations present in the ground truth especially in density and speed variables which deviates in a more pronounced manner. CG-CESN improves the prediction significantly by exploiting the inter-variable dependencies learnt by correlation-guided wiring mechanism. The prediction error is reduced effectively more relative to ESN but discrepancies remain visible near regions of sharp variations. 

The accuracy is further improved by the CGDT-CESN which reconstructs the solar wind dynamics effectively. The prediction of the model closely agrees with the ground truth for all three variables which indicates that the refinement of inter-cluster weights can replicate the interactions between the variables. Most importantly, the model gives stable fit to the data without the signs of over or underfitting as seen in ESN. Hence, CGDT-CESN demonstrates superior capability than that of ESN and CG-CESN which can predict the trends and fluctuations of the solar wind time series.

\subsection{Computational Cost breakdown and Comparison}

\begin{table*}[ht]
\small
\centering
\caption{ Comparison of computational cost for training the model for the Cascade and Solar Wind datasets. Times are reported in seconds as mean $\pm$ standard deviation over 10 realizations. The epochs chosen for cascade and solar wind data sets for CGDT-CESN training is 10 and 5.}
\label{tab:computational_cost}

\renewcommand{\arraystretch}{1.0}

\begin{tabular}{|c|ccc|ccc|}
\hline
\multirow{2}{*}{\textbf{Model}}
&
\multicolumn{3}{c|}{\textbf{Cascade}}
&
\multicolumn{3}{c|}{\textbf{Solar Wind}}
\\
\cline{2-7}

&
\makecell{\textbf{Inter-cluster}\\\textbf{Wiring}}
&
\textbf{Training}
&
\textbf{Total}
&
\makecell{\textbf{Inter-cluster}\\\textbf{Wiring}}
&
\textbf{Training}
&
\textbf{Total}
\\
\hline

ESN & -- & $9.02 \pm 1.09$ & $9.02 \pm 1.09$ & -- & $0.16 \pm 0.03$ & $0.16 \pm 0.03$ \\
\hline

CG-CESN & $341.54 \pm 13.39$ &  $8.72 \pm 1.00$  & $350.26 \pm 13.59$ & $0.30 \pm 0.03$ & $0.49 \pm 0.08$ & $0.79 \pm 0.09$\\
\hline

CGDT-CESN & $341.54 \pm 13.39$ & $125.70 \pm 9.63$  & $467.24 \pm 20.37$ & $0.30 \pm 0.03$ & $0.93 \pm 0.18$ &  $0.95 \pm 0.18$\\
\hline

\end{tabular}
\end{table*}

Since, the prediction performance of the proposed model has been understood on tasks of predicting cascade evolution and solar wind dataset, it is important to understand the computation time required for training the models. All the simulations were performed on the Intel Xeon(R) CPU E5-2630 v2 @ 2.60GHz × 24 processor without utilizing GPU. The comparison of computational cost has been tabulated in Table \ref{tab:computational_cost} since models undergo different computational procedures.

The inter-cluster wiring computation is common to both CG-CESN and CGDT-CESN. The computation time required to calculate the correlation between variables and followed by computing the $p_{out}$, depends on the amount of data used in correlation analysis. It is evident that the volume of data in cascades are larger, which consequently leads to higher computation time as compared to solar wind data set. The training cost of CG-ESN and ESN is more or less similar for both the datasets but the overall compute time when added with the correlation computation, the cost becomes higher for CG-CESN. 

For the CGDT-CESN algorithm, training cost is dependent on the number of epochs chosen. For cascade evolution prediction 10 epochs were used, whereas 5 epochs were enough for the solar wind prediction task. This training methodology imposes additional computational cost but not in a very substantial amount. If the model does not converge within less number of epochs, it will lead to a higher computational expense. But this overhead can be justified by significant prediction performance in addition to the correlation layer than the conventional ESN. The epoch choice and size of the training data decides the computational time. Nevertheless, the proposed algorithm does not cost significant computational expense in terms of time, since it has less number of hyperparameters to be tuned than the traditional RNN architectures.  

\section{Conclusion}
\label{sec:conc}

 An innovative data-driven reservoir design and training methodology for ESN is proposed for handling the challenging task of predicting multivariate time series. This algorithm exploits the cluster-specific specialization of variables, and learns the correlations among variables. The computed correlation helps in establishing connections between clusters. Cluster-specific regression effectively captures a particular variable's pattern. One important thing to note that is that the method for computing the correlation may vary for the data under study and has to be chosen appropriately as there is no universal method which can capture the accurate relationship for all types of process. Hence, up to this computation layer, the algorithm is named Correlation-guided CESN. A novel step includes modelling the inter-cluster connections that undergo recursive training alternating between adjusting the output weights and the inter-cluster weights. This leads to convergence of the prediction satisfactorily near the ground truth than the conventional ESN. The proposed algorithm is therefore named to be Correlation-guided and Dynamics Trained CESN. The algorithm's convergence and echo state property is verified through experiments with correlated and uncorrelated signals.

To test the algorithm, difficult tasks such as the prediction of Information Cascade growth from social media and solar wind datasets, whose underlying process are different and multivariate in nature, are undertaken. Detailed analysis of the Twitter cascade data led to the extraction of variables such as size, depth, breadth, and virality. Extensive experiments were performed to understand the convergence of CGDT-CESN across different training samples, and the effect of the leak rate at different training sample sizes. The application of the models show that correlation-guided wiring of inter-cluster connections improves performance more than conventional ESNs. Superior performance is observed in the CGDT-CESN prediction, due to learning inter-cluster weights in addition to data-driven establishment of inter-cluster connections. This highlights the necessity of adjusting the interaction between the clusters that specialize in adapting to single-variable dynamics, by also learning the influence of weights on the neighboring clusters. Most importantly, even with a small number of training samples, the proposed model, CGDT-CESN performs significantly better than the CG-CESN and ESN at different horizons assuring the effective learning. On application of the model for the solar wind dataset, CG-CESN and CGDT-CESN enhanced the quality of prediction. Further, it is found that the computational cost of CGDT-CESN is not substantially higher than that of the conventional ESN, and the superior prediction performance justifies the additional computational time. 

Modern time series forecasting research focuses models that can provide strong predictive performance with competitive forecasting performance while significantly reducing trainable parameters, training complexity, and computational cost. The proposed algorithm occupies a position at the intersection of forecasting accuracy, and computational efficiency. Overall, in CGDT-CESN, the data-driven inference of coupling topology between the clusters opens new avenues in terms of simultaneously processing the multiple variables offering generalization and flexibility in application for any multivariate time series. The training of the inter-cluster weights indirectly tries to find the exact coupling relationship between the variables apart from being a curve-fitting procedure. Future research may point towards defining the required number of neurons and the required connection density within a cluster for modeling a variable in a data-driven fashion. 

\appendix
\section*{APPENDIX}

\renewcommand{\theequation}{A.\arabic{equation}}
\setcounter{equation}{0}

\section{Theoretical Properties of the CGDT-CESN framework}

\label{app:props_esn}

The theoretical properties of ESN is deep rooted in the echo state property that ensures the reservoir states being uniquely determined by the driving input sequence rather than the initial conditions. ESP assures the reservoir's ability to generate rich dynamical representations of inputs while also preserving the relevant memory of the past information.

The global reservoir state at a time step $n$ is given by,
\begin{equation}
    \mathbf{x}(n)
    = \bigl[\mathbf{x}_1(n)^\top,\,
             \mathbf{x}_2(n)^\top,\,
             \ldots,\,
             \mathbf{x}_K(n)^\top\bigr]^\top
    \in \mathbb{R}^{N},
    \label{eq:global_state}
\end{equation}
where $\mathbf{x}_c(n) \in \mathbb{R}^{N_c}$ is the state of cluster $c$. The reservoir evolves according to the leaky-integrator update (Eq. \ref{eq:cluster_update}).

We define the global state transition function
$F : \mathbb{R}^{N_u} \times \mathbb{R}^{N} \to \mathbb{R}^{N}$ as
\begin{equation}
    F(\mathbf{u},\mathbf{x})
    = (1-\alpha)\,\mathbf{x}
      + \alpha\,\tanh\!\bigl(
          \mathbf{W}\mathbf{x} + \mathbf{W}_{\mathbf{in}}\mathbf{u}
        \bigr),
    \label{eq:F_def}
\end{equation}
so that Eq. (\ref{eq:x(t)}) is written compactly as
$\mathbf{x}(n) = F(\mathbf{u}(n), \mathbf{x}(n-1))$ with full reservoir weight matrix decomposed as $\mathbf{W} = \mathbf{W^{intra}} + \mathbf{W^{inter}}$.
For a sequence of inputs
$\mathbf{s}_N = [\mathbf{u}(1), \mathbf{u}(2), \ldots, \mathbf{u}(N)]$
and an initial state $\mathbf{x} \in \mathbb{R}^{N}$, the
iterated map $\hat{F}(\mathbf{s}_N, \mathbf{x})$ denotes the
reservoir state obtained by driving the network with
$\mathbf{s}_N$ starting from $\mathbf{x}$, defined recursively as
\begin{equation}
    \hat{F}(\mathbf{s}_N,\mathbf{x})
    =
    \begin{cases}
        \mathbf{x}, & \text{if } N = 0,\\[4pt]
        F\!\bigl(\mathbf{u}(N),\;\\
          \hat{F}([\mathbf{u}(1),\ldots,\mathbf{u}(N-1)],
                  \mathbf{x})\bigr), & \text{if } N \geq 1.
    \end{cases}
    \label{eq:Fhat}
\end{equation}
\subsection{The Echo State Property}
\label{sec:esp:definition} 
\begin{definition}[Echo State Property]
\label{def:esp}
The CGDT-CESN governed by Eq. (\ref{eq:x(t)}) is said to
possess the Echo State Property if, for every input
sequence $\mathbf{s}_N$ of length $N$ and for every pair of
initial states $\mathbf{x}, \mathbf{x}' \in \mathbb{R}^{N}$,
\begin{equation}
    \bigl\|\hat{F}(\mathbf{s}_N, \mathbf{x})
           - \hat{F}(\mathbf{s}_N, \mathbf{x}')\bigr\|
    \;\longrightarrow\; 0
    \quad \text{as } N \to \infty.
    \label{eq:esp_def}
\end{equation}
\end{definition}

\subsection{Necessary Condition for the ESP}
\label{sec:esp:necessary}
 
We first establish a necessary condition by linearising the
reservoir dynamics around the zero state.
 
\begin{theorem}[Necessary Condition for ESP]
\label{thm:necessary}
A necessary condition for the CGDT-CESN to hold the ESP is
\begin{equation}
    \rho\!\bigl((1-\alpha)\mathbf{I} + \alpha\mathbf{W}\bigr) < 1.
    \label{eq:necessary}
\end{equation}
\end{theorem}
 
\begin{proof}
The Jacobian of $F$ with respect to $\mathbf{x}$, evaluated at
the zero state $\mathbf{x} = \mathbf{0}$ and under zero input
$\mathbf{u} = \mathbf{0}$, is
\begin{equation}
    \mathbf{J}_F
    = \frac{\partial F}{\partial \mathbf{x}}
      \bigg|_{\mathbf{x}=\mathbf{0},\,\mathbf{u}=\mathbf{0}}
    = (1-\alpha)\mathbf{I}
      + \alpha\,\mathrm{diag}\!\bigl(1 - \tanh^2(\mathbf{0})\bigr)\,
        \mathbf{W}.
    \label{eq:jacobian_general}
\end{equation}
Since $\tanh(0) = 0$, the diagonal derivative matrix evaluates to
the identity, giving
\begin{equation}
    \mathbf{J}_F = (1-\alpha)\mathbf{I} + \alpha\mathbf{W},
    \label{eq:jacobian_zero}
\end{equation} 
where W = \[
\begin{bmatrix}
\mathbf{W^{intra}_{1}} & \mathbf{{W}^{inter}_{(1,2)}}  &\mathbf{{W}^{inter}_{(1,3)}}  & \mathbf{{W}^{inter}_{(1,4)}}  \\
\mathbf{{W}^{inter}_{(2,1)}}  & \mathbf{W^{intra}_{2}} & \mathbf{{W}^{inter}_{(2,3)}}  & \mathbf{{W}^{inter}_{(2,4)}}  \\
\mathbf{{W}^{inter}_{(3,1)}}  & \mathbf{{W}^{inter}_{(3,2)}}  & \mathbf{W^{intra}_{3}} & \mathbf{{W}^{inter}_{(3,4)}}  \\
\mathbf{{W}^{inter}_{(4,1)}}  & \mathbf{{W}^{inter}_{(4,2)}}  & \mathbf{{W}^{inter}_{(4,3)}}  & \mathbf{W^{intra}_{4}}
\end{bmatrix}
\]
with the weight matrix $\mathbf{W}$ represented for $K = 4$ clusters.
The linearised system about the zero state under zero input is
$\mathbf{x}(n) = \mathbf{J}_F\,\mathbf{x}(n-1)$, whose solution
is $\mathbf{x}(n) = \mathbf{J}_F^n\,\mathbf{x}(0)$. For the zero
state to be a stable fixed point it is necessary that
$\mathbf{J}_F^n \to \mathbf{0}$ as $n \to \infty$, which holds if
and only if $\rho(\mathbf{J}_F) < 1$. If
$\rho(\mathbf{J}_F) \geq 1$, there exists an initial state
$\mathbf{x}'$ arbitrarily close to $\mathbf{0}$ whose trajectory
under zero input diverges from $\mathbf{0}$, so the two
trajectories starting from $\mathbf{0}$ and $\mathbf{x}'$ never
converge, violating Definition~\ref{def:esp}. Therefore, the ESP
requires $\rho(\mathbf{J}_F) = \rho\bigl((1-\alpha)\mathbf{I}
+ \alpha\mathbf{W}\bigr) < 1$.
\end{proof}

\begin{remark}
\label{rem:eigenvalue_form}
Let $\lambda_1,\ldots,\lambda_N$ denote the eigenvalues
of~$\mathbf{W}$. The eigenvalues of~$\mathbf{J}_F$ are
$\mu_i=(1{-}\alpha)+\alpha\lambda_i$, and by the triangle
inequality
\begin{equation}
  |\mu_i|
  = \bigl|(1{-}\alpha)+\alpha\lambda_i\bigr|
  \le (1{-}\alpha)+\alpha|\lambda_i|
  \le (1{-}\alpha)+\alpha\,\rho(\mathbf{W}).
  \label{eq:mu_bound}
\end{equation}
 The eigenvalues of
$\mathbf{J}_F$ are $(1-\alpha) + \alpha\lambda_i$, so the
condition in Theorem~\ref{thm:necessary} reduces to
\begin{equation}
    \max_{i}\,|(1-\alpha) + \alpha\lambda_i|
    \le (1-\alpha) + \alpha\,\rho(\mathbf{W})
    < 1.
    \label{eq:necessary_rho}
\end{equation}
For the asymmetric $\mathbf{W}$, $\rho(\mathbf{W}) < 1$
remains sufficient but the ESP may hold even for $\rho(\mathbf{W}) \geq 1$, as confirmed experimentally in Fig. \ref{fig:spectral_esp}.
\end{remark}

\subsection{Sufficient Condition for the ESP}
\label{sec:esp:sufficient} 
We now show that the same spectral radius condition is also
sufficient by establishing that $F$ is a global contraction
mapping.
 
\begin{theorem}
\label{thm:sufficient}
If $L = (1-\alpha) + \alpha\|\mathbf{W}\|_2 < 1$, where $\|\mathbf{W}\|_2 = \sigma_{\max}(W)$ is the largest singular value of the full reservoir weight matrix, then $F$ is a
global contraction mapping and the CGDT-CESN possesses the ESP.
Specifically, for any input sequence $\mathbf{s}_N$, any initial
states $\mathbf{x}, \mathbf{x}' \in \mathbb{R}^{N}$, and any
bounded state space with diameter $D$,
\begin{equation}
    \bigl\|\hat{F}(\mathbf{s}_N, \mathbf{x})
           - \hat{F}(\mathbf{s}_N, \mathbf{x}')\bigr\|
    \leq L^N D
    \;\longrightarrow\; 0
    \quad\text{as } N \to \infty.
    \label{eq:esp_bound}
\end{equation}
\end{theorem}
 
\begin{proof}
We first show that $F$ is Lipschitz continuous in $\mathbf{x}$
with constant $L < 1$. For any $\mathbf{u}$ and any
$\mathbf{x}, \mathbf{x}' \in \mathbb{R}^{N}$:
\begin{align}
\bigl\|F(\mathbf{u}, \mathbf{x})
       - F(\mathbf{u}, \mathbf{x}')\bigr\|
&= \bigl\|(1-\alpha)(\mathbf{x} - \mathbf{x}') \notag \\
&\quad + \alpha\bigl[
    \tanh(\mathbf{W}{x}
          + \mathbf{W}_{\mathrm{in}}\mathbf{u}) \notag \\
&\qquad\;
    - \tanh(\mathbf{W}\mathbf{x}'
            + \mathbf{W}_{\mathrm{in}}\mathbf{u})
\bigr]\bigr\|
\label{eq:proof_step1}\\
&\le (1-\alpha)\|\mathbf{x} - \mathbf{x}'\| \notag \\
&\quad + \alpha\bigl\|
    \tanh(\mathbf{W}\mathbf{x}
          + \mathbf{W}_{\mathrm{in}}\mathbf{u}) \notag \\
&\qquad\;
    - \tanh(\mathbf{W}\mathbf{x}'
            + \mathbf{W}_{\mathrm{in}}\mathbf{u})
\bigr\|
\label{eq:proof_step2}
\end{align}
where~\eqref{eq:proof_step2} follows from the triangle inequality.
Since $\tanh$ is Lipschitz continuous with constant~$1$, i.e.,
$|\tanh(a) - \tanh(b)| \leq |a - b|$ for all $a,b \in \mathbb{R}$
(applied element-wise), the second term satisfies
\begin{equation}
\begin{aligned}
    \bigl\|\tanh(\mathbf{W}\mathbf{x} + \cdots)
           - \tanh(\mathbf{W}\mathbf{x}' + \cdots)\bigr\| 
    &\leq \|\mathbf{W}(\mathbf{x} - \mathbf{x}')\| \\
    &\leq \|\mathbf{W}\|_2\,\|\mathbf{x} - \mathbf{x}'\|,
\end{aligned}
    \label{eq:proof_step3}
\end{equation}
where the last inequality uses the sub-multiplicativity of matrix
norms. Substituting~\eqref{eq:proof_step3}
into~\eqref{eq:proof_step2} gives
\begin{equation}
\begin{aligned}
\|F(\mathbf{u},\mathbf{x})-F(\mathbf{u},\mathbf{x}')\|
&\leq
(1-\alpha)\|\mathbf{x}-\mathbf{x}'\|
+
\alpha \|\mathbf{W}\|_2 \|\mathbf{x}-\mathbf{x}'\| \\
&=
\big[(1-\alpha)+\alpha\|\mathbf{W}\|_2\big]\|\mathbf{x}-\mathbf{x}'\| \\
&=
L\|\mathbf{x}-\mathbf{x}'\|
\end{aligned}
\label{eqn:proof4}
\end{equation}

Since $L < 1$ by assumption, $F$ is a contraction mapping on
$\mathbb{R}^{N}$ for every fixed $\mathbf{u}$. Applying this
bound recursively across all $N$ time steps of the input sequence
$\mathbf{s}_N$:
\begin{align}
    \bigl\|\hat{F}(\mathbf{s}_N, \mathbf{x})
           - \hat{F}(\mathbf{s}_N, \mathbf{x}')\bigr\|
    &\leq L^N\|\mathbf{x} - \mathbf{x}'\|
    \leq L^N D.
    \label{eq:recursive_bound}
\end{align}
Since $0 \leq L< 1$, we have $L^N \to 0$ as $N \to \infty$,
so~\eqref{eq:recursive_bound} satisfies
Definition~\ref{def:esp}, establishing the ESP.
\end{proof}

\section*{Acknowledgements}
H. S. and S. G. gratefully acknowledge the IITM FedEx SMART Center for Supply Chain Modelling, Algorithms, Research and Technology at IIT Madras for providing financial support and research resources for this work. This research was conducted as part of the project titled “Turing Patterns of Misinformation: A Novel Early-Warning System for Fake News,” funded by the IITM FedEx SMART Center at IIT Madras. M. B. acknowledges IITM-Pravartak for financial support. S. G. acknowledges financial support from the Ministry of Education, Government of India.

\bibliographystyle{ieeetr}
\bibliography{reservoir}
\end{document}